\documentclass[11pt]{article}

\usepackage[preprint]{acl}

\usepackage{times}
\usepackage{latexsym}

\usepackage[T1]{fontenc}

\usepackage[utf8]{inputenc}

\usepackage{microtype}

\usepackage{inconsolata}

\usepackage{graphicx}
\usepackage{enumitem}
\usepackage{subfigure}
\usepackage{booktabs} 
\usepackage{array}         
\usepackage{caption}       
\usepackage[table]{xcolor} 
\usepackage{makecell}      
\usepackage{multirow}  
\usepackage{siunitx}
\usepackage[dvipsnames, table]{xcolor}
\usepackage{arydshln}
\usepackage{pifont}
\usepackage{listings}
\usepackage{xcolor}
\definecolor{lightgray}{gray}{0.95}
\usepackage{longtable}
\usepackage{fancyvrb}   
\usepackage{xcolor}     
\usepackage{enumitem}   
\setlist[itemize,enumerate]{noitemsep, topsep=0pt} 
\usepackage{booktabs}  
\usepackage{array}     

\usepackage{amsmath}
\usepackage{amssymb}
\usepackage{mathtools}
\usepackage{amsthm}

\usepackage[capitalize,noabbrev]{cleveref}

\theoremstyle{plain}
\newtheorem{theorem}{Theorem}[section]

\theoremstyle{definition}
\newtheorem{definition}[theorem]{Definition}

\theoremstyle{remark}

\usepackage[textsize=tiny]{todonotes}

\title{RIFT: Repurposing Negative Samples via Reward-Informed Fine-Tuning}

\author{Zehua Liu, Shuqi Liu$^\dagger$, Tao Zhong, Mingxuan Yuan\\
Huawei Noah's Ark Lab\\
\texttt{liuzehua@connect.hku.hk, liu.shuqi1@huawei.com}\\
}

\begin{document}
\maketitle
\begin{abstract}
While Supervised Fine-Tuning (SFT) and Rejection Sampling Fine-Tuning (RFT) are standard for LLM alignment, they either rely on costly expert data or discard valuable negative samples, leading to data inefficiency. To address this, we propose Reward Informed Fine-Tuning (RIFT), a simple yet effective framework that utilizes all self-generated samples. Unlike the hard thresholding of RFT, RIFT repurposes negative trajectories, reweighting the loss with scalar rewards to learn from both the positive and negative trajectories from the model outputs. To overcome the training collapse caused by naive reward integration, where direct multiplication yields an unbounded loss, we introduce a stabilized loss formulation that ensures numerical robustness and optimization efficiency. Extensive experiments on mathematical benchmarks across various base models show that RIFT consistently outperforms RFT. Our results demonstrate that RIFT is a robust and data-efficient alternative for alignment using mixed-quality, self-generated data.
\end{abstract}

{
\let\thefootnote\relax\footnotetext{
$^\dagger$Corresponding author.}
}

\section{Introduction}

The rapid scaling of Large Language Models (LLMs) has made effective post-training adaptation essential \cite{DBLP:journals/corr/abs-2308-10792, DBLP:journals/jmlr/ChungHLZTFL00BW24, DBLP:conf/icml/ChuZYTXSLL025}. Supervised Fine-Tuning (SFT) \cite{DBLP:conf/nips/Ouyang0JAWMZASR22, DBLP:conf/iclr/SanhWRBSACSRDBX22}, which minimizes the negative log-likelihood of expert demonstrations, constitutes the standard approach for aligning models with desired behaviors. However, the efficacy of SFT is heavily dependent upon the availability of high-quality demonstration data, which are generally difficult and costly to curate. More critically, a distributional mismatch between the pre-training data or initial model capabilities and the SFT data can lead to degraded performance, a phenomenon often described as catastrophic forgetting or alignment tax \cite{DBLP:conf/nips/KorbakEKD22, DBLP:journals/corr/abs-2308-08747, DBLP:conf/acl/HuangCWYLSYS24, DBLP:conf/icml/0014LZYZZ025}.


\begin{figure}
    \centering
    \includegraphics[width=1.0\linewidth]{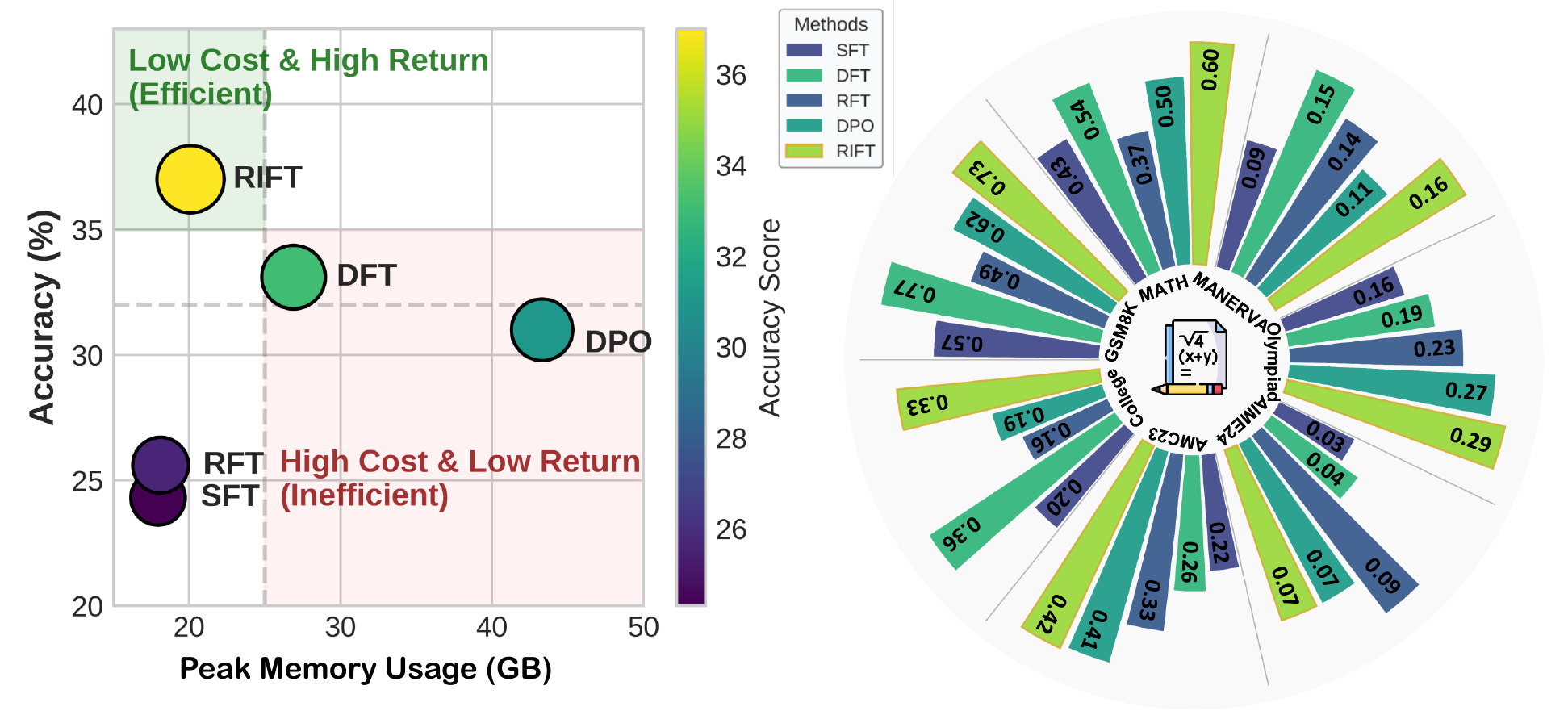}
    \caption{Efficiency and performance of post-training methods for Qwen2.5-Math-1.5B fine-tuned on MATH. \textbf{\textit{Left}}: average accuracy against peak memory utilization (training efficiency); \textbf{\textit{Right}}: per-dataset accuracy (generalization). RIFT surpasses strong baselines in accuracy while requiring less computational memory.}
    \label{fig:gpu_rader_chart}
\end{figure}

To mitigate these data-related limitations in SFT, Rejection Sampling Fine-Tuning (RFT) \cite{yuan_2023_scaling, DBLP:conf/icml/ChenDYJG24} has emerged as a lightweight yet effective alternative. The underlying principle of RFT is straightforward: by sampling multiple responses from a base model to a given prompt and subsequently selecting only those that surpass a predefined quality threshold, one can construct a refined, higher-quality dataset for a subsequent round of SFT. Unlike conventional SFT that relies on pre-constructed static datasets, RFT generates its training data through the model's own sampling process. This self-generation strategy inherently promotes alignment between the data distribution used for fine-tuning and the model's own output distribution. Furthermore, the quality and correctness of the selected data is ensured through an external verification mechanism or a well-defined scoring function.

Despite its simplicity and advantages, the standard RFT paradigm exhibits a critical shortcoming: it discards all sub-threshold (negative) samples outright. This discard policy neglects the potential informational value these samples carry regarding model failure modes. Consequently, it not only wastes computational resources expended during generation, but may also impair the model to learn distinctions between correct and incorrect outputs, thereby limiting its capacity to refine its understanding of subtle errors.

To better utilize all generated responses, including those rejected by quality thresholds in RFT, we propose \textbf{Reward Informed Fine-Tuning (RIFT)}. RIFT constitutes a simple and efficient extension of standard SFT. In contrast to RFT, which discards low-scoring candidates via hard thresholding, RIFT retains every sampled trajectory and assigns it a scalar reward derived from a quality evaluation metric. The RIFT objective modifies the standard negative log-likelihood loss by reweighting each sample's contribution proportionally to its assigned reward. This design ensures that positive-reward samples encourage correct behaviors, whereas negative-reward samples provide reduced or negative gradients.

Nevertheless, a naive integration of the reward signal with the logarithmic probability term presents a significant practical challenge. As we will demonstrate in Section~\ref{sec:methodology}, directly multiplying these components produces a loss function that is unbounded from below, inevitably leading to severe training collapse. To address this fundamental issue, we introduce a principled framework for loss function formulation, which is designed to ensure guaranteed training stability and maintain optimization efficiency.
Figure~\ref{fig:gpu_rader_chart} compares RIFT with strong baselines (e.g., DPO \cite{rafailov_2023_direct}, DFT \cite{wu_2025_generalization}) on Qwen2.5-Math-1.5B: RIFT achieves comparable or superior accuracy at substantially lower peak memory utilization.


Empirically, RIFT delivers consistent and substantial improvements across model scales and alignment settings on mathematical reasoning benchmarks. RIFT outperforms SFT, DFT \cite{wu_2025_generalization}), RFT \cite{yuan_2023_scaling} and DPO \cite{rafailov_2023_direct} in both in-distribution accuracy and out-of-distribution generalization on Qwen2.5-Math (1.5B/7B) \cite{DBLP:journals/corr/abs-2412-15115}, Qwen3-1.7B \cite{DBLP:journals/corr/abs-2505-09388}, and DeepSeek-R1-Distill-Qwen-1.5B \cite{DBLP:journals/corr/abs-2501-12948}. Unlike RFT \cite{yuan_2023_scaling}, which critically depends on strong base models to generate high-quality rollouts, RIFT remains stable and effective even with moderately capable models. In off-policy settings, RIFT consistently surpasses DPO \cite{rafailov_2023_direct} across models. Notably, RIFT eliminates the need for a reference model, offering a simpler and more resource-efficient alternative for alignment.

\section{Related Works} 
\label{sec:related_works}

\paragraph{LLM Post-training}
Supervised Fine-Tuning (SFT) has become the standard post-training paradigm for adapting pretrained models to specific tasks using high-quality, labeled datasets \citep{zhang_2023_instruction, chung_2024_scaling}. Although the availability of high-quality instruction-following datasets \citep{cobbe_2021_gsm8k, mishra_2022_cross, zhou_2023_lima, alpaca} has significantly enhanced the efficacy of SFT, studies \citep{dodge_2020_finetuning, howard_2018_universal, ouyang_2022_training} indicate that SFT often suffers from overfitting and suboptimal generalization.

To mitigate the challenges of SFT data curation, Reinforcement Learning (RL) has emerged as a powerful alternative for post-training. Reinforcement Learning from Human Feedback (RLHF) \citep{ouyang_2022_training} and Reinforcement Learning from Verifiable Reward (RLVR) \citep{guo_2025_deepseek, shao_2024_deepseekmath} are widely used to align models with human preferences and enhance reasoning, supported by algorithms like DPO \citep{rafailov_2023_direct}, 
SimPO \citep{meng_2024_simpo}, 
GRPO \citep{shao_2024_deepseekmath}, and DAPO \cite{yu_2025_dapo}. However, in contrast to SFT, RL-based training is often more complex, necessitating intricate engineering frameworks \citep{zheng_2025_stabilizing}. 


\paragraph{Improving SFT}
Motivated by the success of RL methods, a growing body of research aims to enhance SFT by integrating RL principles. 
RFT \citep{yuan_2023_scaling} utilizes self-generated trajectories filtered for correctness as training data. Other approaches re-frame RL objectives within an SFT framework, such as integrating importance sampling \citep{qin_2025_supervised} or adopting PPO-style clipped surrogates \citep{zhu_2025_proximal}. However, these methods typically require a reference model, imposing a procedural complexity that aligns them more closely with RL paradigms than the simplicity of conventional SFT.

In a parallel line of inquiry, other research focuses on directly modifying the SFT loss without a reference model. For example, DFT \citep{wu_2025_generalization} rescales the SFT objective at each token by its probability. Building on this, \citet{li_2025_beyond} proposed a unified framework for designing loss objectives, demonstrating that DFT can be considered a special case within their formulation. Following this direction, our proposed method, RIFT, refines the loss function to enhance performance while preserving the simplicity of standard SFT.

\begin{figure*}[htbp] 
    \centering 
    \includegraphics[width=0.98\linewidth]{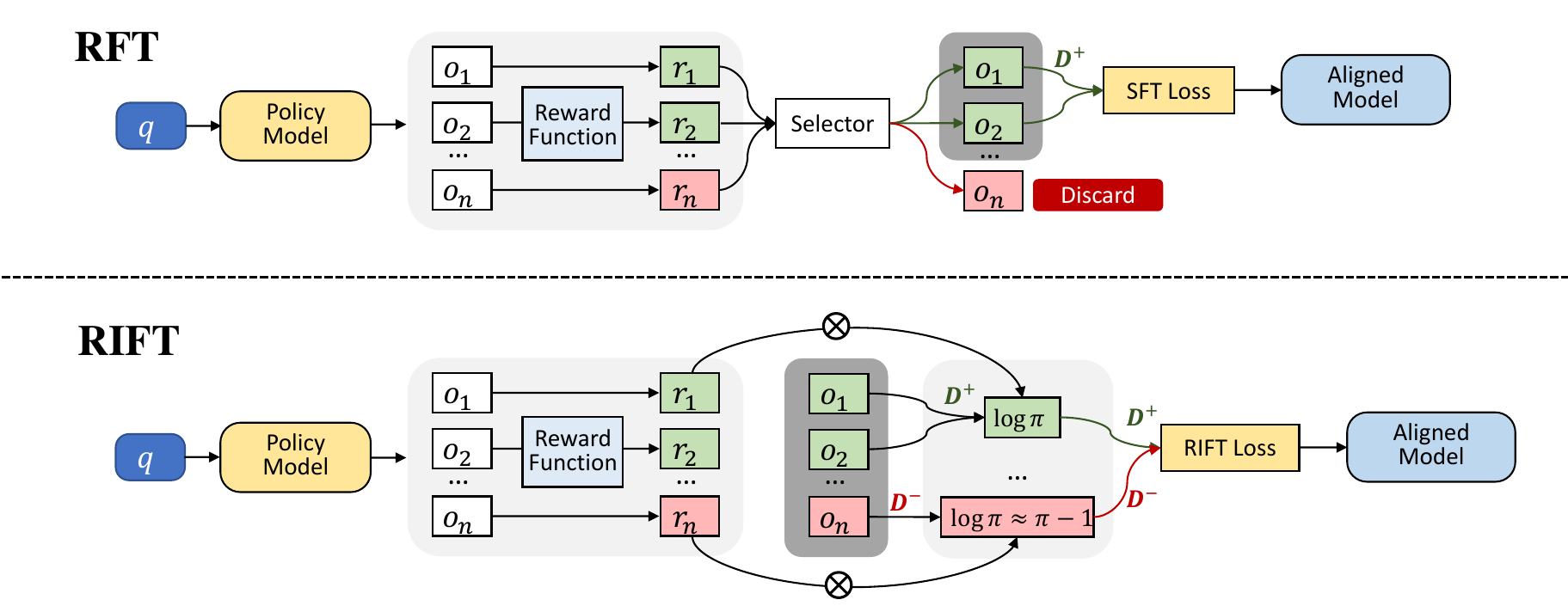}
    \caption{A comparative overview of RFT and RIFT. Unlike RFT rejects negative samples and only trains on positive ones, RIFT repurposes negative samples through a unified reward-informed loss. To ensure stable optimization, a linear surrogate is applied to negative samples to prevent loss collapse.} 
    \label{fig:model_arch}
\end{figure*}

\section{Methodology} \label{sec:methodology}

In this section, we present the theoretical framework and methodology of RIFT (Reward-Informed Fine-Tuning), a generalization of SFT that explicitly leverages mixed-quality demonstrations, i.e., samples with both positive and negative trajectories. An overview of the RIFT framework is depicted in Figure \ref{fig:model_arch}. 
In particular, we address the core challenge of leveraging negative-reward samples without compromising training stability.



\subsection{Preliminaries: Generalized Signed-Weighted Objective}

Standard SFT relies on Maximum Likelihood Estimation (MLE) over high-quality demonstrations, effectively assigning uniform positive weight to all training samples. However, when both positive and negative feedback are available, it is natural to extend MLE by weighting each sample proportionally to its reward: positive rewards encourage likelihood increase, while negative rewards suppress undesirable outputs. This leads to a generalized signed-weighted objective.

Let $\mathcal{D} = \{(x, y, r)\}$ be a dataset where $x$ denotes the input, $y$ the response sampled from the data distribution $\pi_{ref} (\cdot | x)$, and $r: \mathcal{X} \times \mathcal{Y} \to \mathbb{R}$ a scalar reward signal indicating the quality of the response. We partition the response space into positive samples $\mathcal{D}^+ = \{(x, y) \mid r(x, y) > 0\}$ and negative samples $\mathcal{D}^- = \{(x, y) \mid r(x, y) < 0\}$.

\begin{definition}[Naive Signed-Weighted Loss] \label{def:naive_loss}
The naive signed-weighted loss function $\mathcal{L}_{\text{naive}}$ for a parameterized policy $\pi_\theta$ is defined as the expectation of the reward-weighted log-likelihood:
\begin{equation} \label{eq:naive_loss}
\mathcal{L}_{\text{naive}}(\theta) := - \mathbb{E}_{(x, y, r) \sim \mathcal{D}} \left[ r \cdot \log \pi_\theta(y \mid x) \right].
\end{equation}
\end{definition}

The optimization dynamics of Eq.~\eqref{eq:naive_loss} are determined by the sign of $r$:
\begin{itemize}[leftmargin=1.0em, parsep=0pt, topsep=2pt]
\item \textbf{Positive Reinforcement ($r > 0$):} Minimizing $\mathcal{L}_{\text{naive}}$ is equivalent to \textbf{maximizing} $\log \pi_\theta(y|x)$, aligning with the standard SFT objective to promote desirable responses.
\item \textbf{Negative Suppression ($r < 0$):} Minimizing $\mathcal{L}_{\text{naive}}$ is equivalent to \textbf{minimizing} $\log \pi_\theta(y|x)$, theoretically suppressing the generation of undesirable responses.
\end{itemize}

\subsection{Theoretical Analysis of Instability}

While the naive formulation provides a unified view of reinforcement and suppression, it is ill-posed for negative weights due to the asymptotic behavior of the logarithm function.

\begin{theorem}[Gradient Explosion and Unboundedness] \label{thm:collapse}
Consider a negative sample $(x, y) \in \mathcal{D}^-$ with weight $r < 0$. The contribution to the gradient of the loss function $\mathcal{L}_{\text{naive}}$ with respect to the probability $\pi_\theta(y|x)$ is:
\begin{equation}
\frac{\partial \mathcal{L}_{\text{naive}}}{\partial \pi_\theta} = - \frac{r}{\pi_\theta(y|x)}.
\end{equation}
As the model successfully suppresses the negative sample (i.e., $\pi_\theta(y|x) \to 0^+$), the gradient magnitude approaches infinity:
\begin{equation}
\lim_{\pi_\theta \to 0^+} \left| \frac{\partial \mathcal{L}_{\text{naive}}}{\partial \pi_\theta} \right| = \infty.
\end{equation}
Furthermore, the objective function itself is unbounded from below, as $\lim_{p \to 0^+} (-r \log p) = -\infty$ for $r < 0$.
\end{theorem}

Theorem~\ref{thm:collapse} reveals a fundamental optimization pathology: the better the model performs at suppressing a negative sample, the more unstable the gradients become. In practice, this singularity leads to numerical overflow and catastrophic forgetting, where the optimizer focuses excessively on driving infinitesimal probabilities to absolute zero, destroying the feature representations learned from positive data.

\subsection{Reward Informed Fine-Tuning (RIFT)}

Theorem~\ref{thm:collapse} shows a key pathology: stronger suppression of negative samples leads to increasingly unstable gradients.
To address the instability, RIFT replace the logarithmic objective for negative samples with a bounded surrogate.

\subsubsection{Linear Probability Approximation}
Motivated by the first-order Taylor expansion of the logarithm function. For a probability $u \in (0, 1]$, the expansion around $u=1$ is given by:
\begin{equation}
\log u = \sum_{n=1}^\infty \frac{(-1)^{n+1}}{n} (u - 1)^n \approx u - 1.
\end{equation}
Although the linear surrogate $u - 1$ is not accurate near $u = 0$, we adopt it for its stable gradient: unlike $\log u$, its constant derivative avoids explosion as $u \to 0$, ensuring numerical stability while still suppressing negative samples.

\subsubsection{The RIFT Objective}
RIFT decouples positive and negative samples: it retains the log objective for positives (preserving MLE signal) and uses a linear objective for negatives (ensuring stable gradients).

\begin{definition}[RIFT Loss] \label{def:RIFT_loss}
Let $\mathcal{D}^+$ and $\mathcal{D}^-$ be the disjoint sets of positive and negative samples. The RIFT loss function is defined as:
\begin{equation} \label{eq:RIFT_loss}
\begin{aligned}
\mathcal{L}_{\text{RIFT}}(\theta) := & - \mathbb{E}_{(x, y) \sim \mathcal{D}^+} \left[ r(x, y) \cdot \log \pi_\theta (y \mid x) \right] \\
& - \mathbb{E}_{(x, y) \sim \mathcal{D}^-} \left[ r(x, y) \cdot \pi_\theta (y \mid x) \right].
\end{aligned}
\end{equation}
\end{definition}
For $y \in \mathcal{D}^+$, we have $r(x, y) > 0$; thus, minimizing Eq.~\eqref{eq:RIFT_loss} increases $\pi_\theta(y \mid x)$, thereby enhancing positive samples.
In the second term, since $r(x, y) < 0$ for samples in $\mathcal{D}^-$, the term $-r(x, y)$ is positive. Thus, minimizing Eq.~\eqref{eq:RIFT_loss} requires minimizing $\pi_\theta(y|x)$, effectively suppressing the negative samples.

\begin{theorem}[Stability and Properties of RIFT] \label{thm:RIFT_properties}
The RIFT formulation satisfies the following theoretical properties:
\begin{enumerate}[leftmargin=1.5em, parsep=0pt, topsep=2pt]
\item[(\romannumeral1)] \textbf{Boundedness:} Since $\pi_\theta \in [0, 1]$, the loss contribution from any negative sample is bounded in $[r, 0]$, preventing the divergence to $-\infty$.
\item[(\romannumeral2)] \textbf{Reward Lower-Bound Maximization:} Let $\mathcal{J}(\theta) := \mathbb{E}_{y \sim \pi_\theta} [r(x, y)]$ denote the expected reward. Optimizing $\mathcal{L}_{\text{RIFT}}$ can be viewed as maximizing a surrogate lower bound of $\mathcal{J}(\theta)$.
\end{enumerate}
\end{theorem}
By replacing the unbounded logarithmic penalty with a bounded linear penalty, RIFT provides stable incorporation of negative samples, ensuring that the suppression of undesirable content does not compromise the stability of the fine-tuning process.
\section{Experiments} \label{sec:exps}

\subsection{Experiment Details}
\paragraph{Base Models and Off-Policy Data Construction}
We evaluate RIFT against four established baselines. We first include supervised and rejection-sampling-based methods: SFT, DFT \cite{wu_2025_generalization}, and RFT \cite{yuan_2023_scaling}. Furthermore, as models can benefit from contrasting correct and incorrect outcomes, we also compare against DPO \cite{rafailov_2023_direct}, a representative off-policy RL method. 
Experiments are conducted on Qwen2.5-Math (1.5B, 7B) \cite{DBLP:journals/corr/abs-2412-15115}, Qwen3-1.7B \cite{DBLP:journals/corr/abs-2505-09388}, and DeepSeek-R1-Distill-Qwen-1.5B \cite{DBLP:journals/corr/abs-2501-12948}, with the Qwen3 variant evaluated in non-thinking mode.


\begin{table}[htbp]
\centering
\large
\resizebox{\columnwidth}{!}{
\renewcommand{\arraystretch}{0.95}
\begin{tabular}{lccccc}
\toprule
Model &  \# Num. & \# Total & \parbox{1.2cm}{\centering \# Pos. \\ (r$>$0)} & \parbox{1.4cm}{\centering \# Neg. \\ (r$<$0)} & \% Pos.  \\
\midrule
\multicolumn{6}{c}{\textbf{\textit{Source: MATH Dataset}}} \\ 
\midrule
Qwen-2.5-Math-1.5B & 3,000 & 24,000 & 15,941 & 8,059 & 66.4\%  \\
Qwen-2.5-Math-7B & 3,000  & 24,000 & 16,933 & 7,067 & 70.6\%  \\
Qwen-3-1.7B & 3,000 & 24,000 & 20,386 & 3,614 & 84.9\%  \\
\midrule
\multicolumn{6}{c}{\textbf{\textit{Source: NuminaMath Dataset}}} \\ 
\midrule
Qwen-2.5-Math-1.5B & 4,000 & 32,000 & 11,235 & 20,765 & 35.1\%  \\
Qwen-2.5-Math-7B & 4,000  & 32,000 & 10,581 & 21,419 & 33.1\%  \\
Qwen-3-1.7B & 4,000 & 32,000 & 20,352 & 11,648 & 63.6\%  \\
\bottomrule
\end{tabular}
} 
\caption{Training data statistics across models and datasets, including counts of positive and negative samples and the positive sample ratio.
}
\vspace{-2mm}
\label{tab:data_stats}
\end{table}

\begin{table*}[t]
\centering
\small
\resizebox{\textwidth}{!}{
\renewcommand{\arraystretch}{0.95}
\begin{tabular}{llcccccccc}
\toprule
\textbf{Model} & \textbf{Method} & \textbf{GSM8K} & \textbf{MATH} & \textbf{Minerva} & \textbf{Olympiad} & \textbf{AIME24} & \textbf{AMC23} & \textbf{College} & \textbf{Avg.} \\
\midrule
\multicolumn{10}{c}{\cellcolor{white!10}\textbf{\textit{Post-Train on MATH Dataset}}} \\
\midrule
\multirow{7}{*}{\textbf{Qwen-2.5-Math-1.5B}}
& Base & 42.6 & 35.6 & 9.7 & 22.6 & 7.1 & 31.9 & 8.2 & 22.5 \\
\cdashline{2-10}
& SFT & 57.0 & 42.9 & 9.3 & 16.1 & 3.3 & 21.9 & 19.9 & 24.3 \\
& DFT & \textbf{76.8} & 53.9 & 15.6 & 19.1 & 4.2 & 25.6 & \textbf{36.4} & \underline{33.1} \\
& RFT & 48.8 & 37.2 & 13.5 & 22.5 & \textbf{8.8} & 33.4 & 15.2 & 25.6 \\
& DPO & 61.8 & 50.3 & 11.3 & 26.7 & 7.1 & 41.2 & 18.6 & 31.0 \\
\cdashline{2-10}
& \cellcolor[HTML]{E6E6FA}\textbf{RIFT} & \cellcolor[HTML]{E6E6FA}72.6 & \cellcolor[HTML]{E6E6FA}\textbf{59.6} & \cellcolor[HTML]{E6E6FA}\textbf{15.8} & \cellcolor[HTML]{E6E6FA}\textbf{28.8} & \cellcolor[HTML]{E6E6FA}7.1 & \cellcolor[HTML]{E6E6FA}\textbf{41.9} & \cellcolor[HTML]{E6E6FA}33.3 & \cellcolor[HTML]{E6E6FA}\textbf{\underline{37.0}} \textcolor{red!50!black}{\textbf{(+11.4)}} \\
\midrule
\multirow{7}{*}{\textbf{Qwen-2.5-Math-7B}}
& Base & 54.8 & 50.3 & 12.2 & 16.4 & 12.1 & 36.9 & 20.5 & 29.0 \\
\cdashline{2-10}
& SFT & 67.0 & 48.9 & 10.8 & 16.6 & 2.9 & 25.6 & 26.9 & 28.4 \\
& DFT & 83.3 & 58.5 & 16.9 & 20.9 & 4.6 & 33.8 & 35.2 & 36.2 \\
& RFT & 79.3 & 72.1 & 21.3 & 35.7 & 11.2 & \textbf{59.1} & 42.0 & \underline{45.8} \\
& DPO & 62.0 & 61.7 & \textbf{26.3} & 31.3 & 16.2 & 50.3 & 36.8 & 40.7 \\
\cdashline{2-10}
& \cellcolor[HTML]{E6E6FA}\textbf{RIFT} & \cellcolor[HTML]{E6E6FA}\textbf{84.6} & \cellcolor[HTML]{E6E6FA}\textbf{74.0} & \cellcolor[HTML]{E6E6FA}25.4 & \cellcolor[HTML]{E6E6FA}\textbf{36.1} & \cellcolor[HTML]{E6E6FA}\textbf{17.9} & \cellcolor[HTML]{E6E6FA}58.8 & \cellcolor[HTML]{E6E6FA}\textbf{43.8} & \cellcolor[HTML]{E6E6FA}\textbf{\underline{48.7}} \textcolor{red!50!black}{\textbf{(+2.9)}} \\
\midrule
\multirow{7}{*}{\textbf{\shortstack{Qwen-3-1.7B \\ (Non-thinking mode)}}}
& Base & 77.0 & 42.3 & 19.1 & 13.4 & 1.2 & 22.5 & 30.8 & 29.5 \\
& SFT & 80.0 & 50.1 & 22.5 & 17.7 & 1.2 & 28.4 & 33.8 & 33.4 \\
& DFT & 84.4 & 57.0 & 27.7 & 21.7 & 4.2 & 31.2 & 36.3 & 37.5 \\
& RFT & 87.0 & 67.3 & 30.1 & 27.5 & 5.0 & 39.1 & 41.1 & \underline{42.4} \\
& DPO & 86.6 & 66.0 & 26.1 & 26.5 & 6.7 & \textbf{43.4} & 35.7 & 41.6 \\
\cdashline{2-10}
& \cellcolor[HTML]{E6E6FA}\textbf{RIFT} & \cellcolor[HTML]{E6E6FA}\textbf{87.3} & \cellcolor[HTML]{E6E6FA}\textbf{69.3} & \cellcolor[HTML]{E6E6FA}\textbf{32.7} & \cellcolor[HTML]{E6E6FA}\textbf{29.7} & \cellcolor[HTML]{E6E6FA}\textbf{7.9} & \cellcolor[HTML]{E6E6FA}41.6 & \cellcolor[HTML]{E6E6FA}\textbf{41.5} & \cellcolor[HTML]{E6E6FA}\textbf{\underline{44.3}} \textcolor{red!50!black}{\textbf{(+1.9)}} \\
\midrule
\multicolumn{10}{c}{\cellcolor{white!10}\textbf{\textit{Post-Train on NuminaMath Dataset}}} \\
\midrule
\multirow{7}{*}{\textbf{Qwen-2.5-Math-1.5B}}
& Base & 42.6 & 35.6 & 9.7 & 22.6 & 7.1 & 31.9 & 8.2 & 22.5 \\
\cdashline{2-10}
& SFT & 67.5 & 51.4 & 11.8 & 18.5 & 5.0 & 29.4 & 30.9 & 30.6 \\
& DFT & \textbf{77.4} & 57.8 & 17.3 & 25.2 & 6.7 & 31.2 & \textbf{34.2} & 35.7 \\
& RFT & 69.7 & 62.1 & 15.2 & \textbf{28.6} & 5.2 & 37.8 & 32.7 & \underline{35.9} \\
& DPO & 73.5 & 61.9 & 15.8 & 27.6 & 3.3 & 37.7 & 31.1 & 35.8 \\
\cdashline{2-10}
& \cellcolor[HTML]{E6E6FA}\textbf{RIFT} & \cellcolor[HTML]{E6E6FA} 75.2 & \cellcolor[HTML]{E6E6FA}\textbf{62.4} & \cellcolor[HTML]{E6E6FA}\textbf{18.1} & \cellcolor[HTML]{E6E6FA}27.8 & \cellcolor[HTML]{E6E6FA}\textbf{7.1} & \cellcolor[HTML]{E6E6FA}\textbf{40.0} & \cellcolor[HTML]{E6E6FA}33.5 & \cellcolor[HTML]{E6E6FA}\textbf{\underline{37.7}} \textcolor{red!50!black}{\textbf{(+1.4)}} \\
\midrule
\multirow{7}{*}{\textbf{Qwen-2.5-Math-7B}}
& Base & 54.8 & 50.3 & 12.2 & 16.4 & 12.1 & 36.9 & 20.5 & 29.0 \\
\cdashline{2-10}
& SFT & 71.1 & 60.9 & 21.8 & 32.9 & 9.2 & 43.4 & 37.0 & 39.5 \\
& DFT & \textbf{87.0} & 70.6 & 26.1 & \textbf{34.7} & 7.5 & 44.7 & 37.9 & 44.1  \\
& RFT & 83.3 & 69.8 & 21.3 & 31.3 & 11.2 & 58.8 & 42.0 & 45.4 \\
& DPO & 84.5 & 71.4 & 27.2 & 32.9 & 16.2 & 56.1 & 38.4 & \underline{46.7} \\
\cdashline{2-10}
& \cellcolor[HTML]{E6E6FA}\textbf{RIFT} & \cellcolor[HTML]{E6E6FA}86.3 & \cellcolor[HTML]{E6E6FA}\textbf{74.7} & \cellcolor[HTML]{E6E6FA}\textbf{28.9} & \cellcolor[HTML]{E6E6FA}34.1 & \cellcolor[HTML]{E6E6FA}\textbf{17.1} & \cellcolor[HTML]{E6E6FA}\textbf{62.2} & \cellcolor[HTML]{E6E6FA}38.6 & \cellcolor[HTML]{E6E6FA}\textbf{\underline{48.8}} \textcolor{red!50!black}{\textbf{(+3.4)}} \\
\midrule
\multirow{7}{*}{\textbf{\shortstack{Qwen-3-1.7B \\ (Non-thinking mode)}}}
& Base & 77.0 & 42.3 & 19.1 & 13.4 & 1.2 & 22.5 & 30.8 & 29.5 \\
\cdashline{2-10}
& SFT & 84.7 & 62.2 & 22.9 & 24.3 & 2.5 & 37.8 & 34.1 & 38.4 \\
& DFT & 87.6 & 69.9 & 30.5 & 28.3 & 3.3 & 42.5 & 36.2 & 42.6 \\
& RFT & 86.4 & 62.3 & 25.5 & 24.3 & 3.3 & 36.6 & 34.6 & \underline{39.0} \\
& DPO & 86.8 & 66.7 & 27.7 & 26.3 & 6.7 & 40.8 & 36.0 & 41.6 \\
\cdashline{2-10}
& \cellcolor[HTML]{E6E6FA}\textbf{RIFT} & \cellcolor[HTML]{E6E6FA}\textbf{88.2} & \cellcolor[HTML]{E6E6FA}\textbf{69.2} & \cellcolor[HTML]{E6E6FA}\textbf{28.2} & \cellcolor[HTML]{E6E6FA}\textbf{28.7} & \cellcolor[HTML]{E6E6FA}\textbf{3.3} & \cellcolor[HTML]{E6E6FA}\textbf{46.6} & \cellcolor[HTML]{E6E6FA}\textbf{36.3} & \cellcolor[HTML]{E6E6FA}\textbf{\underline{42.9}} \textcolor{red!50!black}{\textbf{(+3.9)}} \\
\bottomrule
\end{tabular}
}
\caption{Mean@8 accuracy (\%) on 7 mathematical benchmarks. Best results are in \textbf{bold}. \textcolor{red!50!black}{\textbf{(+)}} indicates the absolute improvement of RIFT compared to RFT.}
\vspace{-5mm}
\label{tab:main_mean8}
\end{table*}

To construct off-policy training data, we curate two buffers from 3,000 randomly sampled MATH \cite{hendrycks_2021_measuring} and 4,000 NuminaMath \cite{numina_math_datasets} problems. For each problem, the base model generates 8 candidate responses, each assigned a reward based on final-answer correctness: positive reward for correct responses and negative for incorrect ones.
Following findings in MGPO \cite{xu2025tinymodelbiglogic}, we set larger magnitude reward ($+1.0$) for positive responses than for negative ones ($-0.2$) to emphasize successful reasoning traces. 
We analyze sensitivity to the negative reward in Section~\ref{subsec:reward_robustness}. The final buffers consist of $(x, y, r)$ triplets, with statistics in Table~\ref{tab:data_stats}.
Regarding learning strategies: SFT and DFT train on the seed problems with their ground-truth solutions; RFT uses only positive-reward responses, discarding all negative ones; DPO forms preference pairs by comparing model responses to ground-truth solutions, preferring the response when correct and the ground truth otherwise; In contrast, RIFT leverages the full training buffer, requiring neither data filtering nor explicit preference pairing.


\vspace{-2mm}
\paragraph{Implementation Details and Hyperparameter Settings}

We implement baselines using the built-in recipes of the MS-Swift \citep{zhao_2024_swift} framework, while RIFT is implemented via TRL \citep{vonwerra_2022_trl}. 
Unless otherwise specified, we adopt the default configurations provided by MS-Swift.
For candidate response generation, we sample 8 candidates per problem with a temperature of $0.7$ and a maximum sequence length of 4,096. During inference, all models maintain these settings with a top-$p$ of $0.8$ and a fixed random seed (0) for reproducibility. 
Optimization is carried out using the AdamW \cite{DBLP:conf/iclr/LoshchilovH19} optimizer coupled with a cosine learning rate scheduler featuring a 5\% warmup phase. 
The learning rate is set to $1 \times 10^{-5}$ for SFT and RFT, and a more conservative $2 \times 10^{-6}$ for RIFT and DPO to ensure stability during preference-based updates; 
All experiments are conducted with a global batch size of 64 over three epochs.

\begin{table*}[t]
\centering
\small
\resizebox{\textwidth}{!}{
\renewcommand{\arraystretch}{0.95}
\begin{tabular}{llcccccccc} 
\toprule
\textbf{Model} & \textbf{Method} & \textbf{GSM8K} & \textbf{MATH} & \textbf{Minerva} & \textbf{Olympiad} & \textbf{AIME24} & \textbf{AMC23} & \textbf{College} & \textbf{Avg.} \\
\midrule
\multicolumn{10}{c}{\cellcolor{white!10}\textbf{\textit{Post-Train on Math Dataset}}} \\
\midrule
\multirow{7}{*}{\textbf{Qwen-2.5-Math-1.5B}}
& Base & 88.0 & 75.1 & 32.0 & 46.5 & 23.3 & 67.5 & 30.1 & 51.8 \\
\cdashline{2-10}
& SFT & 93.6 & 80.5 & 29.0 & 43.7 & 16.7 & 60.0 & 49.0 & 53.2 \\
& DFT & \textbf{94.5} & 76.7 & \textbf{38.2} & 43.0 & 20.0 & 55.0 & \textbf{52.4} & 54.3 \\
& RFT & 87.0 & 67.3 & 30.1 & 27.5 & 5.0 & 39.1 & 41.1 & 42.4 \\
& DPO & 92.9 & 83.2 & 33.8 & 48.7 & \textbf{33.3} & 72.5 & 45.4 & \underline{58.5} \\
\cdashline{2-10}
& \cellcolor[HTML]{E6E6FA}\textbf{RIFT} & \cellcolor[HTML]{E6E6FA}93.9 & \cellcolor[HTML]{E6E6FA}\textbf{85.9} & \cellcolor[HTML]{E6E6FA}37.1 & \cellcolor[HTML]{E6E6FA}\textbf{51.7} & \cellcolor[HTML]{E6E6FA}30.0 & \cellcolor[HTML]{E6E6FA}\textbf{80.0} & \cellcolor[HTML]{E6E6FA}51.9 & \cellcolor[HTML]{E6E6FA}\textbf{\underline{61.5}} \textcolor{red!50!black}{\textbf{(+19.1)}} \\
\midrule
\multirow{7}{*}{\textbf{Qwen-2.5-Math-7B}}
& Base & 92.1 & 83.6 & 36.4 & 42.5 & 30.0 & 70.0 & 48.2 & 50.7 \\
\cdashline{2-10}
&  SFT & 95.8 & 82.8 & 33.5 & 43.0 & 16.7 & 62.5 & 49.6 & 54.8 \\
&  DFT & 92.3 & 72.6 & 33.1 & 39.1 & 13.3 & 60.0 & 47.3 & 51.1 \\
& RFT & 95.5 & \textbf{90.2} & 46.7 & 58.1 & 33.3 & 85.0 & 54.2 & \underline{66.1} \\
& DPO & 94.6 & 88.9 & \textbf{52.2} & 56.1 & 33.3 & 82.5 & 55.1 & \underline{66.1} \\
\cdashline{2-10}
& \cellcolor[HTML]{E6E6FA}\textbf{RIFT} & \cellcolor[HTML]{E6E6FA}\textbf{96.4} & \cellcolor[HTML]{E6E6FA}90.1 & \cellcolor[HTML]{E6E6FA}49.6 & \cellcolor[HTML]{E6E6FA}\textbf{59.3} & \cellcolor[HTML]{E6E6FA}\textbf{36.7} & \cellcolor[HTML]{E6E6FA}\textbf{85.0} & \cellcolor[HTML]{E6E6FA}\textbf{55.5} & \cellcolor[HTML]{E6E6FA}\textbf{\underline{67.5}} \textcolor{red!50!black}{\textbf{(+1.4)}} \\
\midrule
\multirow{6}{*}{\textbf{\shortstack{Qwen-3-1.7B \\ (Non-thinking mode)}}}
& Base & 90.6 & 67.1 & 34.9 & 31.4 & 10.0 & 45.0 & 41.9 & 45.8 \\
\cdashline{2-10}
& SFT & 92.7 & 74.1 & 38.2 & 36.4 & 10.0 & 47.5 & 44.8 & 49.1 \\
& DFT & 94.4 & 80.2 & 43.8 & 41.8 & 13.3 & 55.0 & 47.5 & 53.7 \\
& RFT & 94.3 & 84.6 & 42.3 & 40.0 & 20.0 & 60.0 & 48.2 & 55.6 \\
& DPO & 94.7 & 82.6 & 39.3 & 41.6 & \textbf{26.7} & \textbf{70.0} & 41.2 & \underline{56.7} \\
\cdashline{2-10}
& \cellcolor[HTML]{E6E6FA}\textbf{RIFT} & \cellcolor[HTML]{E6E6FA}\textbf{94.8} & \cellcolor[HTML]{E6E6FA}\textbf{85.6} & \cellcolor[HTML]{E6E6FA}\textbf{45.6} & \cellcolor[HTML]{E6E6FA}\textbf{45.8} & \cellcolor[HTML]{E6E6FA}20.0 & \cellcolor[HTML]{E6E6FA}65.0 & \cellcolor[HTML]{E6E6FA}\textbf{48.4} & \cellcolor[HTML]{E6E6FA}\textbf{\underline{57.9}}  \textcolor{red!50!black}{\textbf{(+2.3)}}\\
\midrule
\multicolumn{10}{c}{\cellcolor{white!10}\textbf{\textit{Post-Train on NuminaMath Dataset}}} \\
\midrule
\multirow{7}{*}{\textbf{Qwen-2.5-Math-1.5B}}
& Base & 88.0 & 75.1 & 32.0 & 46.5 & 23.3 & 67.5 & 30.1 & 51.8 \\
\cdashline{2-10}
& SFT & 94.0 & 82.8 & 36.8 & 45.3 & 16.7 & 70.0 & 54.2 & 57.1 \\
& DFT & 93.3 & 85.3 & 40.1 & 50.4 & 16.7 & 62.5 & 52.0 & 57.2 \\
& RFT & 93.6 & 86.1 & 38.2 & \textbf{52.4} & 23.3 & 75.0 & 45.4 & 59.1 \\
& DPO & 93.6 & 85.9 & 39.3 & 51.7 & 23.3 & 77.5 & 46.1 & \underline{59.6} \\
\cdashline{2-10}
& \cellcolor[HTML]{E6E6FA}\textbf{RIFT} & \cellcolor[HTML]{E6E6FA}\textbf{94.2}  & \cellcolor[HTML]{E6E6FA}\textbf{86.4} & \cellcolor[HTML]{E6E6FA}\textbf{41.9} & \cellcolor[HTML]{E6E6FA}49.5 & \cellcolor[HTML]{E6E6FA}\textbf{26.7} & \cellcolor[HTML]{E6E6FA}\textbf{80.0} & \cellcolor[HTML]{E6E6FA}\textbf{45.8} & \cellcolor[HTML]{E6E6FA}\textbf{\underline{60.6}} \textcolor{red!50!black}{\textbf{(+1.5)}}\\
\midrule
\multirow{7}{*}{\textbf{Qwen-2.5-Math-7B}}
& Base & 92.1 & 83.6 & 36.4 & 42.5 & 30.0 & 70.0 & 48.2 & 50.7 \\
\cdashline{2-10}
& SFT & 96.0 & 89.5 & 45.2 & \textbf{60.3} & 23.3 & 80.0 & 56.4 & 64.4 \\
& DFT & 91.7 & 81.8 & 37.5 & 48.7 & 16.7 & 62.5 & 42.7 & 54.5 \\
& RFT & 95.5 & 90.1 & 49.6 & 56.1 & 33.3 & 82.5 & 51.9 & 65.6 \\
& DPO & 95.8 & 90.2 & 52.2 & 58.1 & 36.7 & 85.0 & 54.2 & \underline{67.5} \\
\cdashline{2-10}
\cdashline{2-10}
& \cellcolor[HTML]{E6E6FA}\textbf{RIFT} & \cellcolor[HTML]{E6E6FA}\textbf{96.3}  & \cellcolor[HTML]{E6E6FA}\textbf{90.6} & \cellcolor[HTML]{E6E6FA}\textbf{53.3} & \cellcolor[HTML]{E6E6FA}58.4 & \cellcolor[HTML]{E6E6FA}\textbf{43.3} & \cellcolor[HTML]{E6E6FA}\textbf{87.5} & \cellcolor[HTML]{E6E6FA}48.3 & \cellcolor[HTML]{E6E6FA}\textbf{\underline{68.2}} \textcolor{red!50!black}{\textbf{(+2.6)}} \\
\midrule
\multirow{6}{*}{\textbf{\shortstack{Qwen-3-1.7B \\ (Non-thinking mode)}}}
& Base & 90.6 & 67.1 & 34.9 & 31.4 & 10.0 & 45.0 & 41.9 & 45.8 \\
\cdashline{2-10}
& SFT & 93.5 & 81.3 & 34.9 & 41.2 & 10.0 & 67.5 & 40.8 & 52.7 \\
& DFT & 93.4 & 81.6 & 39.3 & \textbf{43.4} & \textbf{16.7} & 67.5 & 41.0 & \underline{54.7} \\
& RFT & 94.8 & 81.3 & 38.6 & 41.0 & 13.3 & 62.5 & 40.7 & 53.2 \\
& DPO & 94.1 & 81.6 & 38.6 & 41.6 & 13.3 & 62.5 & 41.2 & 53.3 \\
\cdashline{2-10}
& \cellcolor[HTML]{E6E6FA}\textbf{RIFT} & \cellcolor[HTML]{E6E6FA}\textbf{94.9} & \cellcolor[HTML]{E6E6FA}\textbf{85.6} & \cellcolor[HTML]{E6E6FA}\textbf{42.6} & \cellcolor[HTML]{E6E6FA}43.1 & \cellcolor[HTML]{E6E6FA}13.3 & \cellcolor[HTML]{E6E6FA}\textbf{70.0} & \cellcolor[HTML]{E6E6FA}\textbf{41.8} & \cellcolor[HTML]{E6E6FA}\textbf{\underline{55.9}} \textcolor{red!50!black}{\textbf{(+2.7)}} \\
\bottomrule
\end{tabular}
}
\caption{Pass@8 accuracy (\%) on 7 mathematical benchmarks. Best results are in \textbf{bold}. \textcolor{red!50!black}{\textbf{(+)}} indicates the absolute improvement of RIFT compared to RFT.}
\label{tab:main_pass8}
\vspace{-5mm}
\end{table*}

\vspace{-2mm}
\paragraph{Evaluation Benchmarks and Metrics}
Following prior studies, we adopt mathematical tasks as our primary testbed. Specifically, we evaluate on seven math benchmarks: GSM8K \cite{cobbe_2021_gsm8k}, MATH \cite{hendrycks_2021_measuring}, Minerva Math \cite{DBLP:conf/nips/LewkowyczADDMRS22}, Olympiad Bench \cite{DBLP:conf/nips/HuangWXLZXFYCY024}, AIME 2024 \cite{maa2024aime}, AMC 2023 \cite{maa2023amc}, and College Math \cite{DBLP:conf/iclr/HendrycksBBZMSS21}. 
We use the standardized Qwen2.5-Math-Eval pipeline \cite{yang_2024_qwen2} and report Mean@8 and Pass@8.

\subsection{Main Results}
\paragraph{Mean@8 Performance}
Table~\ref{tab:main_mean8} reports Mean@8 accuracy across seven mathematical reasoning benchmarks. RIFT consistently achieves the highest average performance in all settings, surpassing SFT, DFT, RFT, and DPO without requiring explicit preference pairs or data filtering. Our analysis yields the following findings:

\textbf{(1) SFT and DFT: limited OOD generalization.} 
Trained solely on MATH, SFT and DFT underperform the base model on harder OOD tasks (e.g., DFT: 19.1 vs. 22.6 on Olympiad; 4.6 vs. 12.1 on AIME24), but recover with NuminaMath whose distribution better aligns with the benchmarks. In contrast, by leveraging mixed-reward responses, RIFT consistently outperforms the base across all benchmarks, even under MATH-only training.

\textbf{(2) RFT scales with model capacity.} 
On Qwen-Math-1.5B, RFT underperforms DPO (25.6 vs. 31.0), but surpasses it on Qwen-Math-7B (45.8 vs. 40.7), indicating that RFT requires sufficient high-quality positive samples for self-improvement. However, RIFT stabilizes the refinement process even when self-generation quality is moderate.

\begin{table}[htbp]
\centering
\large
\resizebox{\columnwidth}{!}{
\renewcommand{\arraystretch}{0.95}
\begin{tabular}{lccc}
\toprule
Model & \# Num. & \# Mixed-Reward Num. & \% Mixed \\
\midrule
\multicolumn{4}{c}{\textbf{\textit{Source: MATH Dataset}}} \\ 
\midrule
Qwen-2.5-Math-1.5B & 3,000 & 2,541 & 84.7\% \\
Qwen-2.5-Math-7B   & 3,000 & 1,947 & 64.9\% \\
Qwen-3-1.7B & 3,000 & 971 & 32.4\% \\
\midrule
\multicolumn{4}{c}{\textbf{\textit{Source: NuminaMATH Dataset}}} \\ 
\midrule
Qwen-2.5-Math-1.5B & 4,000 & 2,060 & 51.5\% \\
Qwen-2.5-Math-7B   & 4,000 & 2,305 & 57.6\% \\
Qwen-3-1.7B & 4,000 & 3,462 & 86.6\% \\
\bottomrule
\end{tabular}
} 
\caption{Fraction of problems with mixed correct and incorrect responses (out of 8) per model and dataset.}
\vspace{-3mm}
\label{tab:data_stats_mixed}
\end{table}


\begin{table*}[!htbp]
\centering
\small
\resizebox{\textwidth}{!}{
\begin{tabular}{llcccccccc}
\toprule
\textbf{Model} & \textbf{Method} & \textbf{GSM8K} & \textbf{MATH} & \textbf{Minerva} & \textbf{Olympiad} & \textbf{AIME24} & \textbf{AMC23} & \textbf{College} & \textbf{Avg.} \\
\midrule
\multirow{6}{*}{\textbf{\shortstack{DeepSeek-R1-Distill-Qwen-1.5B \\ (Mean@8)}}}
& Base & 80.5 & 70.4 & 19.5 & 30.2 & 12.9 & 47.5 & 39.6 & 43.0 \\
\cdashline{2-10}
& SFT & 50.5 & 38.2 & 10.8 & 9.7 & 0.4 & 14.4 & 25.2 & 21.3 \\
& DFT & 76.3 & 70.6 & \textbf{23.5} & 30.1 & 13.3 & 42.5 & 37.8 & 42.0 \\
& RFT & 63.6 & 63.5 & 14.7 & 27.4 & 13.8 & 48.1 & 33.9 & 37.9 \\
& DPO & 80.9 & 69.2 & 16.9 & 28.8 & \textbf{14.6} & \textbf{50.3} & 39.2 & \underline{42.8} \\
\cdashline{2-10}
& \cellcolor[HTML]{E6E6FA}\textbf{RIFT} & \cellcolor[HTML]{E6E6FA}\textbf{82.1} & \cellcolor[HTML]{E6E6FA}\textbf{71.1} & \cellcolor[HTML]{E6E6FA}22.3 & \cellcolor[HTML]{E6E6FA}\textbf{30.3} & \cellcolor[HTML]{E6E6FA}13.3 & \cellcolor[HTML]{E6E6FA}48.8 & \cellcolor[HTML]{E6E6FA}\textbf{40.1} & \cellcolor[HTML]{E6E6FA}\textbf{\underline{44.0}} \textcolor{red!50!black}{\textbf{(+6.1)}} \\
\midrule
\multirow{6}{*}{\textbf{\shortstack{DeepSeek-R1-Distill-Qwen-1.5B \\ (Pass@8)}}}
& Base & 95.1 & 89.8 & 36.4 & 40.9 & 33.3 & 70.0 & 50.2 & 59.4 \\
\cdashline{2-10}
& SFT & 85.0 & 71.2 & 31.2 & 31.0 & 3.3 & 52.5 & 46.1 & 45.8 \\
& DFT & 93.8 & 89.4 & \textbf{43.4} & \textbf{50.7} & 30.0 & 77.5 & 48.2 & \underline{61.9} \\
& RFT & 92.6 & 89.2 & 31.2 & 47.9 & 30.0 & 75.0 & 50.0 & 59.4 \\
& DPO & 94.9 & 89.1 & 32.7 & 47.0 & 30.0 & 77.5 & 49.5 & 60.1 \\
\cdashline{2-10}
& \cellcolor[HTML]{E6E6FA}\textbf{RIFT} & \cellcolor[HTML]{E6E6FA}\textbf{95.2} & \cellcolor[HTML]{E6E6FA}\textbf{89.9} & \cellcolor[HTML]{E6E6FA}40.4 & \cellcolor[HTML]{E6E6FA}\textbf{50.7} & \cellcolor[HTML]{E6E6FA}\textbf{33.3} & \cellcolor[HTML]{E6E6FA}\textbf{82.5} & \cellcolor[HTML]{E6E6FA}\textbf{50.3} & \cellcolor[HTML]{E6E6FA}\textbf{\underline{63.2}} \textcolor{red!50!black}{\textbf{(+3.8)}} \\
\bottomrule
\end{tabular}
}
\caption{Mean@8 and Pass@8 accuracy (\%) on 7 mathematical benchmarks for DeepSeek-R1-Qwen-1.5B model. 
Best results are in \textbf{bold}. \textcolor{red!50!black}{\textbf{(+)}} indicates the absolute improvement of RIFT compared to RFT.
}
\label{tab:reasoner}
\end{table*}

\begin{table*}[!htbp]
\centering
\small
\resizebox{\textwidth}{!}{
\begin{tabular}{lcccccccc}
\toprule
\textbf{Reward Method} & \textbf{GSM8K} & \textbf{MATH} & \textbf{Minerva} & \textbf{Olympiad} & \textbf{AIME24} & \textbf{AMC23} & \textbf{College} & \textbf{Avg.} \\
\midrule
\rowcolor[HTML]{E6E6FA} \multicolumn{9}{c}{\textbf{\textit{Group Normalization Reward}}} \\
GPG-Mean & 68.8 $\pm$ 0.76 & 57.0 $\pm$ 0.49 & 14.8 $\pm$ 1.07 & 27.7 $\pm$ 2.02 & \textbf{10.0 $\pm$ 2.69} & 43.3 $\pm$ 4.25 & 25.4 $\pm$ 0.29 & 35.3 $\pm$ 1.65 \\
GPG-Scaled & \textbf{70.2 $\pm$ 0.67} & 57.5 $\pm$ 0.22 & 15.5 $\pm$ 0.19 & \textbf{28.6 $\pm$ 0.53} & \textbf{10.0 $\pm$ 2.69} & 43.8 $\pm$ 7.07 & \textbf{26.4 $\pm$ 0.91} & 36.0 $\pm$ 1.75\\
Gaussian Norm & 69.4 $\pm$ 0.46 & \textbf{57.8 $\pm$ 0.21} & \textbf{16.5 $\pm$ 2.36} & 28.6 $\pm$ 0.62 & \textbf{10.0 $\pm$ 2.69} & \textbf{48.3 $\pm$ 5.14} &  25.4 $\pm$ 0.66 & \textbf{36.6 $\pm$ 1.73} \\
\midrule
\rowcolor[HTML]{E6E6FA} \multicolumn{9}{c}{\textbf{\textit{Constant Negative Reward}}} \\
$r_{neg}=-0.2$ & \textbf{73.2 $\pm$ 0.29} & \textbf{59.6 $\pm$ 0.22} & \textbf{18.5 $\pm$ 1.14} & \textbf{28.8 $\pm$ 0.70} & \textbf{11.1 $\pm$ 1.91} & 43.3 $\pm$ 2.89 & 27.6 $\pm$ 0.29 & \textbf{37.5 $\pm$ 1.06} \\
$r_{neg}=-0.5$ & 72.0 $\pm$ 0.96 & 58.8 $\pm$ 0.91 & 12.4 $\pm$ 1.52 & 28.2 $\pm$ 1.31 & 10.0 $\pm$ 3.30 & \textbf{46.7 $\pm$ 2.29} & \textbf{29.6 $\pm$ 0.40} & 36.8 $\pm$ 1.53 \\
$r_{neg}=-0.8$ & 72.8 $\pm$ 1.00 & 59.4 $\pm$ 0.76 & 17.6 $\pm$ 1.15 & 27.3 $\pm$ 1.06 & 10.0 $\pm$ 3.82 & 45.0 $\pm$ 2.90 & 28.6 $\pm$ 0.38 & 37.2 $\pm$ 1.58 \\
\bottomrule
\end{tabular}
}
\caption{Performance comparison of reward methods across 7 mathematical reasoning benchmarks. Mean score ($\pm$ standard deviation) over three runs is reported. Best results are \textbf{bolded}.}
\vspace{-5mm}
\label{tab:reward_robust}
\end{table*}

\textbf{(3) DPO exhibits greater robustness.}
DPO achieves consistent improvements over the base model via pairwise preference learning (e.g. +8.5 on Qwen2.5-1.5B, +11.7 on Qwen2.5-7B, +8.0 on Qwen3-1.7B, trained on MATH), but its advantage over RFT diminishes with stronger base models.
RIFT, in contrast, dominates across all scales by explicitly modeling reward signals, proving more effective than pair preference alignment alone.

(4) \textbf{Mixed-reward responses drive RIFT gains.}
As Table~\ref{tab:data_stats_mixed} shows, on MATH the mixed-reward rate (rollouts with both correct and incorrect responses) drops with model scale (84.7\% to 32.4\%), and the gain of RIFT over RFT declines accordingly (+11.4 to +1.9). On NuminaMath, however, larger models yield higher mixed-reward rates (86.6\% for Qwen3-1.7B) and the largest RIFT gains (+3.9), confirming that RIFT benefits most when correct and incorrect responses coexist.

\paragraph{Pass@8 Performance}
Table~\ref{tab:main_pass8} reports Pass@8 (probability of more than 1 correct solution in 8 generations).  RIFT consistently achieves the highest and most stable Pass@8 across all settings.
\textbf{(1) RFT prioritizes correctness at the cost of solution diversity.}
While RFT achieves competitive Mean@8, its Pass@8 consistently lags behind DPO, as RFT relies solely on correct rollouts, yielding high-quality but low-diversity solutions. In contrast, DPO achieves higher Pass@8 by contrasting correct and incorrect outcomes, forcing the model to explore a wider strategy space. 
\textbf{(2) RIFT outperforms implicit pairwise comparisons.}
RIFT further surpasses DPO in Pass@8 (+3.0 on Qwen-2.5-Math-1.5B, +1.4 on Qwen-2.5-Math-7B, and +1.2 on Qwen-3 1.7B), showing that the explicit use of the reward signal enables more effective exploration than implicit pairwise comparison.



\begin{table*}[]
\small
\centering
\resizebox{\textwidth}{!}{%
\renewcommand{\arraystretch}{0.95}
\begin{tabular}{lcccccc}
\toprule
\multirow{2}{*}{\textbf{Method}} & \multicolumn{2}{c}{\textbf{Qwen-2.5-Math-1.5B}} & \multicolumn{2}{c}{\textbf{DeepSeek-R1-Distill-Qwen-1.5B}} & \multicolumn{2}{c}{\textbf{Qwen3-1.7B}} \\
\cmidrule(r){2-3} \cmidrule(r){4-5} \cmidrule(r){6-7}
& Peak Memory usage (GB) $\downarrow$ 
& Acc (\%) $\uparrow$ 
& Peak Memory usage (GB) $\downarrow$ 
& Acc (\%) $\uparrow$ 
& Peak Memory usage (GB) $\downarrow$ 
& Acc (\%) $\uparrow$ \\
\midrule
SFT  & \textbf{17.95} & 24.3 & \textbf{15.59} & 21.3 & \textbf{19.40} & 33.4 \\
DFT  & 26.90 & \underline{33.1} & 21.57 & 42.0 & 21.04 & 37.5 \\
RFT  & \underline{18.12} & 25.6 & \underline{18.15} & 37.9 & \underline{19.47} & \underline{42.4} \\
DPO  & 43.31 & 31.0 & 43.36 & \underline{42.8} & 41.24 & 41.6 \\
\hdashline
\rowcolor[HTML]{F2F2FF}
\textbf{RIFT} & 20.10 \textcolor{green!50!black}{\textbf{(+1.98)}} & \textbf{37.0} \textcolor{red!70!black}{\textbf{(+11.4)}} & 22.28 \textcolor{green!50!black}{\textbf{(+4.13)}} & \textbf{44.0} \textcolor{red!70!black}{\textbf{(+6.1)}} & 22.00 \textcolor{green!50!black}{\textbf{(+2.53)}} & \textbf{44.3} \textcolor{red!70!black}{\textbf{(+1.9)}} \\
\bottomrule
\end{tabular}%
}
\caption{Computational efficiency and performance trade-off. Accuracy (Acc) represents the mean@8 score averaged across 7 mathematical benchmarks. \textcolor{red!70!black}{(+)} and \textcolor{green!50!black}{(+)} indicate the absolute improvement in Acc and the absolute increase in peak computational memory of RIFT compared to RFT.}
\vspace{-5mm}
\label{tab:gpu_acc_tradeoff}
\end{table*}

\vspace{-1mm}
\subsection{Extending to Reasoner Model}
\vspace{-1mm}
To evaluate RIFT on models with intrinsic reasoning, we adopt DeepSeek-R1-Distill-Qwen-1.5B~\cite{DBLP:journals/corr/abs-2501-12948}, a distilled reasoner that generates explicit reflective traces, unlike non-thinking models that depend on prompt-based thinking. Given the extended reasoning traces, we set the maximum length of self-generated responses to 8,192. As Table~\ref{tab:reasoner} shows, this setting reveals key alignment failures in baseline methods.

\textbf{(1) SFT breaks the built-in reasoning.}
SFT on MATH severely degrades performance (21.3 vs. Base 43.0), indicating direct SFT on non-reflective data actively degrades the inherent capacity for step-by-step thinking of the reasoner model.
\textbf{(2) RFT and DPO only approach base performance.}
While RFT recovers Pass@8 (59.4), reaching levels comparable to the base model, it simultaneously degrades the Mean@8 (37.9 vs. Base 43.0). DPO, in contrast, maintains a comparable Mean@8 while achieving slightly higher Pass@8. 
(3) \textbf{RIFT delivers robust gains.} 
RIFT achieves the highest performance across all metrics: 44.0 Mean@8 (+1.0 over Base) and 63.2 Pass@8 (+3.8 over Base), representing the largest absolute improvement observed among all tested methods. Notably, RIFT significantly outperforms strong baselines like DPO by +1.2 in Mean@8 and +3.1 in Pass@8. 

\subsection{Generalization to General Preference Alignment}

While RIFT is primarily motivated by tasks with verifiable rewards (e.g., mathematical reasoning), we further evaluate its generalizability on broader preference alignment. We fine-tune the Llama3.2-3B\cite{grattafiori2024llama} model on the UltraFeedback \cite{cuiultrafeedback} dataset and compare RIFT against the standard DPO baseline across three diverse benchmarks: IFEval\cite{zhou2023instruction}, TruthfulQA\cite{lin2022truthfulqa}, and HellaSwag\cite{zellers2019hellaswag}.

As shown in Table~\ref{tab:rlhf}, RIFT consistently outperforms DPO across all three benchmarks, demonstrating its effectiveness beyond verifiable reward tasks. 
Notably, RIFT achieves the improvement with approximately 50\% less GPU memory usage than DPO. 
We observe that the performance gains on RLHF tasks are more modest compared to mathematical reasoning. 
We attribute this to the inherent noise in preference datasets (e.g., annotator bias, model-generated preferences), which provide less clean supervision signals than the deterministic rewards available in math tasks. 
Nevertheless, RIFT's consistent improvements confirm its robustness even under noisy feedback conditions.

\begin{table}[h]
\centering
\small
\resizebox{0.49\textwidth}{!}{
\renewcommand{\arraystretch}{0.95}
\begin{tabular}{lcccc}
\toprule
Llama-3.2-3B & IFEval & TruthfulQA & Hellaswag & Avg. \\
\midrule
Base & 12.23 & 19.42 & 52.32 & 27.99 \\
DPO & 18.34 & 22.12 & 53.74 & 31.40 \\
RIFT & \textbf{18.63} & \textbf{24.94} & \textbf{55.89} & \textbf{33.15} \\
\bottomrule
\end{tabular}
} 
\caption{Zero-shot evaluation on general preference alignment tasks. RIFT outperforms DPO while using $\sim$50\% less memory.}
\vspace{-3mm}
\label{tab:rlhf}
\end{table}

\section{In-depth Analysis}

\subsection{Reward Strategy and Robustness Analysis}
\label{subsec:reward_robustness}

To assess how the reward design impacts the effectiveness of RIFT, we compare two distinct classes of reward strategies: (1) constant negative rewards ($r_{\text{neg}} \in \{-0.2, -0.5, -0.8\}$) and (2) group-wise normalization, which rescales rewards per problem based on its self-generated response set. We evaluate three normalization variants, $\hat{r}$ represents the reward $r$ after normalization:

\begin{itemize}[leftmargin=1.0em, parsep=0pt, topsep=2pt]
    \item Gaussian Normalization: Standardizes rewards within each problem’s solution group:
    {\setlength{\abovedisplayskip}{5pt}
    \setlength{\belowdisplayskip}{5pt}
    \begin{equation}
        \hat{r} = {(r - \mu)}/{\sigma}
    \end{equation}
    }where $\mu$ and $\sigma$ are the mean and standard deviation of rewards for that problem.
    \item GPG Normalization (Mean-Centered): Adapts the advantage formulation of GPG~\cite{DBLP:journals/corr/abs-2504-02546}:
    {\setlength{\abovedisplayskip}{5pt}
    \setlength{\belowdisplayskip}{5pt}
    \begin{equation}
        \hat{r} = \alpha \cdot (r - \mu), \quad \alpha = {N^{+}} / {N},
    \end{equation}}$N^{+}$ and $N$ denote the numbers of correct and total responses in each group. The scaling factor $\alpha$ acts as an adaptive gain controller, amplifying the learning signal for problems with higher success rates.
    \item GPG Normalization (Raw-Scaled): Preserves the original reward sign and relative magnitude:
    {\setlength{\abovedisplayskip}{5pt}
    \setlength{\belowdisplayskip}{5pt}
    \begin{equation}
        \hat{r} = \alpha \cdot r, \quad \alpha = N^{+} / {N}.
    \end{equation}}
\end{itemize}

Table~\ref{tab:reward_robust} evaluates different reward mechanisms: 
\textbf{(1) Constant negative reward outperforms dynamic normalization.} Surprisingly, simple constant negative rewards consistently surpass group normalization methods, suggesting absolute reward can be more effective than relative intra-group rewards. 
\textbf{(2) RIFT is remarkably robust to negative reward magnitude.} Average performance remain stable within the $[-0.2, -0.8]$ range, indicating a consistent rejection signal enables stable and superior performance over RFT.



\begin{table*}[t]
\centering
\tiny
\resizebox{\textwidth}{!}{
\renewcommand{\arraystretch}{0.95}
\begin{tabular}{llcccccccc}
\toprule
\textbf{Model} & \textbf{Method} & \textbf{GSM8K} & \textbf{MATH} & \textbf{Minerva} & \textbf{Olympiad} & \textbf{AIME24} & \textbf{AMC23} & \textbf{College} & \textbf{Avg.} \\
\midrule
\multicolumn{10}{c}{\cellcolor{white!10}\textbf{\textit{Post-Train on MATH Dataset}}} \\
\midrule
\multirow{8}{*}{\textbf{\shortstack{Qwen-2.5-Math-1.5B \\ (Mean@8)}}}
& Base & 42.6 & 35.6 & 9.7 & 22.6 & 7.1 & 31.9 & 8.2 & 22.5 \\
\cdashline{2-10}
& SFT & 57.0 & 42.9 & 9.3 & 16.1 & 3.3 & 21.9 & 19.9 & 24.3 \\
& DFT & \textbf{76.8} & 53.9 & 15.6 & 19.1 & 4.2 & 25.6 & \textbf{36.4} & 33.1 \\
& RFT & 48.8 & 37.2 & 13.5 & 22.5 & \textbf{8.8} & 33.4 & 15.2 & 25.6 \\
& DPO & 61.8 & 50.3 & 11.3 & 26.7 & 7.1 & 41.2 & 18.6 & 31.0 \\
& SimPO & 65.8 & 50.0 & \textbf{18.0} & 18.4 & 5.8 & 20.6 & 35.0 & 30.5 \\
& KTO & 69.9 & 58.5 & 17.2 & 28.3 & 8.8 & 40.9 & 26.9 & \underline{35.8} \\
\cdashline{2-10}
& \cellcolor[HTML]{E6E6FA}\textbf{RIFT} & \cellcolor[HTML]{E6E6FA}72.6 & \cellcolor[HTML]{E6E6FA}\textbf{59.6} & \cellcolor[HTML]{E6E6FA}15.8 & \cellcolor[HTML]{E6E6FA}\textbf{28.8} & \cellcolor[HTML]{E6E6FA}7.1 & \cellcolor[HTML]{E6E6FA}\textbf{41.9} & \cellcolor[HTML]{E6E6FA}33.3 & \cellcolor[HTML]{E6E6FA}\textbf{\underline{37.0}} \\
\midrule
\multirow{8}{*}{\textbf{\shortstack{Qwen-2.5-Math-1.5B \\ (Pass@8)}}}
& Base & 88.0 & 75.1 & 32.0 & 46.5 & 23.3 & 67.5 & 30.1 & 51.8 \\
\cdashline{2-10}
& SFT & 93.6 & 80.5 & 29.0 & 43.7 & 16.7 & 60.0 & 49.0 & 53.2 \\
& DFT & \textbf{94.5} & 76.7 & 38.2 & 43.0 & 20.0 & 55.0 & 52.4 & 54.3 \\
& RFT & 87.0 & 67.3 & 30.1 & 27.5 & 5.0 & 39.1 & 41.1 & 42.4 \\
& DPO & 92.9 & 83.2 & 33.8 & 48.7 & \textbf{33.3} & 72.5 & 45.4 & 58.5 \\
& SimPO & \textbf{94.5} & 81.8 & 40.8 & 45.6 & 16.7 & 62.5 & \textbf{55.5} & 56.8 \\
& KTO & 93.5 & 82.1 & \textbf{44.9} & 50.2 & 30.0 & 77.5 & 42.8 & \underline{60.1}\\
\cdashline{2-10}
& \cellcolor[HTML]{E6E6FA}\textbf{RIFT} & \cellcolor[HTML]{E6E6FA}93.9 & \cellcolor[HTML]{E6E6FA}\textbf{85.9} & \cellcolor[HTML]{E6E6FA}37.1 & \cellcolor[HTML]{E6E6FA}\textbf{51.7} & \cellcolor[HTML]{E6E6FA}30.0 & \cellcolor[HTML]{E6E6FA}\textbf{80.0} & \cellcolor[HTML]{E6E6FA}51.9 & \cellcolor[HTML]{E6E6FA}\textbf{\underline{61.5}} \\
\bottomrule
\end{tabular}
}
\caption{Comparison with SimPO and KTO on Qwen2.5-Math-1.5B. Results show Mean@8 and Pass@8 accuracy (\%) across 7 mathematical benchmarks.}
\vspace{-2mm}
\label{tab:simpo_kto_mean8_pass8}
\end{table*}

\begin{table*}[h]
\centering
\tiny
\resizebox{0.8\textwidth}{!}{
\renewcommand{\arraystretch}{0.95}
\begin{tabular}{lccccccc}
\toprule
Qwen-2.5-Math-1.5B & SFT & DFT & RFT & DPO & SimPO & KTO & RIFT \\
\midrule
Peak Memory usage (GB) $\downarrow$ & \textbf{17.95} & 26.90 & \underline{18.12} & 43.31 & 21.40 & 42.93 & 20.10 \\
Average Accuracy (\%) $\uparrow$ & 24.30 & 33.10 & 25.60 & 31.00 & 30.50 & \underline{35.80} & \textbf{37.00} \\
\bottomrule
\end{tabular}
} 
\caption{Efficiency and performance trade-off against SimPO and KTO on Qwen2.5-Math-1.5B. Results show peak GPU memory usage and average mathematical reasoning accuracy.}
\vspace{-3mm}
\label{tab:simpo_kto_trade_off}
\end{table*}

\subsection{Computational Efficiency}

We evaluate the computational efficiency of RIFT by measuring peak computational memory usage during training alongside average accuracy across seven benchmarks. As demonstrated in Table \ref{tab:gpu_acc_tradeoff}, 
RIFT maintains a highly favorable performance-efficiency trade-off across all backbones. Specifically, while DPO incurs substantial memory overhead (exceeding 41 GB) due to the necessity of loading a reference model, RIFT requires only 20.1 to 22.3 GB. This represents nearly a 50\% reduction in peak VRAM usage compared to DPO, while consistently yielding higher accuracy (e.g., 44.3\% vs. 41.6\% on Qwen3-1.7B). Furthermore, the computational cost of RIFT is comparable to that of SFT and RFT, introducing only marginal overhead.

\subsection{Comparison with Efficient Alignment Baselines}

To further evaluate the efficiency and effectiveness of RIFT, we extend our comparison to include recent reference-free and low-memory alignment methods, specifically SimPO \cite{meng2024simpo} and KTO \cite{ethayarajh2024kto}. We conduct experiments using the Qwen2.5-Math-1.5B model to assess both performance gains and computational overhead.
As illustrated in Table \ref{tab:simpo_kto_mean8_pass8} and \ref{tab:simpo_kto_trade_off}, RIFT achieves a superior trade-off between mathematical reasoning performance and computational overhead. Although SimPO achieves memory savings similar to RIFT by removing the reference model, its performance gains over DPO are marginal, trailing RIFT by 6.5 points. Similarly, while KTO shows a notable improvement over DPO (+4.8 points), it still lags behind RIFT by 1.2 points. Crucially, unlike KTO which maintains a high memory usage, RIFT delivers superior performance with significantly higher resource efficiency.




\section{Conclusion}
\vspace{-1mm}

We propose RIFT, a simple yet effective post-training framework that leverages the full distribution of self-generated samples.
Unlike RFT, which discards valuable negative samples via hard thresholding, RIFT leverages both high- and low-reward trajectories. 
To ensure optimization stability, we introduce a principled loss formulation that effectively prevents training collapse during reward integration.
Extensive evaluation across seven mathematical reasoning benchmarks shows that RIFT consistently outperforms established baselines. This demonstrates that explicitly learning from mixed-quality data allows models to better internalize correct reasoning and failure modes. 
As a robust and data-efficient alignment method, RIFT enables scalable self-improvement without reliance on extensive expert-labeled data.

\section*{Limitations}

While RIFT demonstrates substantial gains in data efficiency and performance, several limitations remain for further refinement. 

First, as a reward-informed framework, RIFT is designed to effectively bridge the gap between reward signals and policy optimization. While it maximizes the utility of self-generated feedback, the performance upper-bound is naturally influenced by the discriminative power of the reward source. This is a shared challenge across all reward-driven alignment methodologies. Our results show that RIFT is robust to mixed-quality data, and exploring uncertainty-aware reward weighting to further mitigate potential feedback noise remains a compelling direction. 

Second, our evaluation primarily focuses on verifiable reasoning tasks characterized by objective success criteria and deterministic outcomes. Although mathematical benchmarks provide a high-fidelity environment to validate the core mechanics of RIFT, extending this framework to subjective or open-ended generative domains remains an open challenge. In such contexts, where correctness is harder to define, and the rewards are more difficult to measure, which might require a more complex reward setup.

Finally, RIFT currently treats each sample as a single unit. In complex, multi-step problems, a model might fail just because of one small mistake in a long, mostly correct path. At the moment, we do not look inside the steps to find these almost correct parts. Future versions of RIFT could use step-by-step rewards to learn from these partial successes, which could help the model improve even faster.

Regarding safety and ethical considerations, while the base models may occasionally generate reasoning errors, the practical risk is minimal as all model outputs are utilized strictly as internal training signals rather than for real-world deployment. Furthermore, all evaluations are conducted on standard mathematical benchmarks using objective metrics, ensuring a controlled experimental environment. In terms of manuscript preparation, an AI assistant is employed to enhance the linguistic clarity. However, all AI-generated suggestions are carefully reviewed and refined by the authors, ensuring that the final manuscript accurately reflects our own judgment and contains no harmful or misleading content.

\bibliography{custom.bib}




\begin{table*}[!t]
\centering
\resizebox{\textwidth}{!}{
\begin{tabular}{llcccccccc}
\toprule
\textbf{Model} & \textbf{Num. $K$} & \textbf{GSM8K} & \textbf{MATH} & \textbf{Minerva} & \textbf{Olympiad} & \textbf{AIME24} & \textbf{AMC23} & \textbf{College} & \textbf{Avg.} \\
\midrule
\multirow{6}{*}{\textbf{\shortstack{Qwen-2.5-Math-1.5B \\ (Mean@3)}}}
& Base & 41.6 & 35.2 & 9.3 & 22.4 & 10.0 & 32.5 & 8.3 & 22.8 \\
\cdashline{2-10}
& 2 & 57.7 & 46.4 & 11.3 & 25.1 & \textbf{11.1} & 39.2 & 14.2 & 29.3 \\
& 4 & 62.0 & 50.2 & 13.5 & 27.1 & 7.8 & 41.7 & 18.8 & 31.6 \\
& 8 & \textbf{73.2} & \textbf{59.6} & 18.5 & \textbf{28.8} & \textbf{11.1} & \textbf{43.3} & 27.6 & \textbf{37.5} \\
& 16 & 71.9 & 59.2 & \textbf{19.4} & 28.5 & 5.6 & 39.2 & \textbf{28.1} & 36.0 \\
\midrule
\multirow{6}{*}{\textbf{\shortstack{Qwen-2.5-Math-1.5B \\ (Pass@3)}}}
& Base & 69.0 & 58.1 & 13.5 & 37.2 & 20.0 & 52.5 & 17.8 & 38.3 \\
\cdashline{2-10}
& 2 & 81.3 & 69.4 & 23.5 & 37.6 & \textbf{26.7} & 60.0 & 26.2 & 46.4 \\
& 4 & 83.2 & 71.4 & 27.6 & 41.5 & 20.0 & 62.5 & 31.1 & 48.2 \\
& 8 & \textbf{88.6} & \textbf{77.3} & \textbf{33.5} & \textbf{41.9} & 20.0 & \textbf{70.0} & \textbf{38.6} & \textbf{52.8} \\
& 16 & \textbf{88.6} & 76.6 & 32.0 & 41.2 & 13.3 & 57.5 & 38.3 & 49.8 \\
\bottomrule
\end{tabular}
}
\caption{Ablation study on the number of self-generated responses per problem $K$ in RIFT training for Qwen-2.5-Math-1.5B.
For each $K \in \{ 2,4,8,16 \}$, we sample $K$ responses per problem from the base model.
We report Mean@3 and Pass@3 accuracy (\%) across 7 mathematical benchmarks.}
\label{tab:rollout_ablation}
\end{table*}

\newpage

\appendix
\section{Extended Analysis and Empirical Explorations}
\label{appendix:extended_analysis}


In this section, we provide additional empirical evidence and in-depth analyses to further elucidate the behavior of RIFT. Specifically, we investigate how the scale of self-generated data influences the alignment performance and examine the underlying dynamics of policy convergence facilitated by our framework.

\subsection{Ablation Study on the Number of Self-Generated Responses}
\label{sec:addition_exp}

To evaluate how the quantity of self-generated data affects the alignment performance of RIFT, we conduct an ablation study by varying the number of sampled responses per problem ($K$). While the main experiments utilize $K=8$, we investigate the performance across $K \in \{2, 4, 8, 16\}$, using Qwen-2.5-Math-1.5B trained on MATH. All other hyperparameters remain consistent with the main training setup.

As shown in Table~\ref{tab:rollout_ablation}, increasing $K$ from 2 to 8 consistently improves both Mean@3 (+8.2 points) and Pass@3 (+6.4 points), with $K=8$ achieving the highest scores across nearly all benchmarks. However, further increasing $K$ to 16 leads to a noticeable drop particularly in Pass@3 ($-$3.0 points), suggesting that more samples do not always translate to better alignment.

The performance trend is intriguing because the underlying data statistics, reported in Tables~\ref{tab:rollout_data_stats} and \ref{tab:rollout_data_stats_mixed}, show little variation across $K$: the proportion of positive responses remains stable at around 66.4\%, and the fraction of problems with mixed reward signals plateaus near 85\% for $K \geq 4$. In other words, simply generating more responses does not significantly alter the overall composition of the training data.

The resolution lies in the quality of exploration, not just its quantity. While the aggregate statistics appear similar, larger $K$ enables richer coverage of the solution space to capture a wider variety of reasoning patterns, including subtle failure modes. At $K=8$, this diversity is sufficient for RIFT to learn robust distinctions between correct and flawed reasoning, without overwhelming the model with redundant or low-signal trajectories. By contrast, $K=16$ introduces diminishing returns: the marginal gain in mixed-reward problems (from 84.7\% to 92.3\%) comes at the cost of increased noise, which disproportionately harms the reliability of top predictions, as reflected in the sharper decline of Pass@3.

\begin{table}[htbp]
\centering
\resizebox{\columnwidth}{!}{
\renewcommand{\arraystretch}{0.95}
\begin{tabular}{lccccc}
\toprule
$K$ &  \# Num. & \# Total & \parbox{1.2cm}{\centering \# Pos. \\ (r$>$0)} & \parbox{1.4cm}{\centering \# Neg. \\ (r$<$0)} & \% Pos.  \\
\midrule
2 & 3,000 & 6,000 & 3,992 & 2,008 & 66.5\%  \\
4 & 3,000 & 12,000 & 7,964 & 4,036 & 66.4\%  \\
8 & 3,000 & 24,000 & 15,941 & 8,059 & 66.4\%  \\
16 & 3,000 & 48,000 & 31,943 & 16,057 & 66.5\%  \\
\bottomrule
\end{tabular}
} 
\caption{Statistics of self-generated rollouts for 3,000 problems from the MATH dataset, sampled using the Qwen-2.5-Math-1.5B base model with $K\in \{2,4,8,16 \}$ responses per problem.}
\vspace{-2mm}
\label{tab:rollout_data_stats}
\end{table}


As shown in Table \ref{tab:rollout_data_stats}, the overall proportion of positive samples (\% Pos.) remains remarkably consistent at approximately $66.5\%$ as $K$ increases from 2 to 16. 
More importantly, Table \ref{tab:rollout_data_stats_mixed} shows that the percentage of problems with mixed-reward responses (containing both positive and negative outcomes) scales with $K$, rising from 78.5\% to 92.3\%. This trend demonstrates that a larger $K$ effectively augments intra-problem diversity.

\begin{table}[htbp]
\centering
\small
\resizebox{\columnwidth}{!}{
\renewcommand{\arraystretch}{0.95}
\begin{tabular}{lccc}
\toprule
$K$ & \# Num. & \# Mixed-Reward Num. & \% Mixed \\
\midrule
2 & 3,000 & 2,356 & 78.5\% \\
4 & 3,000 & 2,548 & 84.9\% \\
8 & 3,000 & 2,541 & 84.7\% \\
16 & 3,000 & 2,768 & 92.3\% \\
\bottomrule
\end{tabular}
} 
\caption{Proportion of problems containing both positive- and negative-reward responses in their $K$ self-generated responses.}
\vspace{-3mm}
\label{tab:rollout_data_stats_mixed}
\end{table}

The performance drop when increasing from $K=8$ to $K=16$ may stem from overfitting to noisy negative patterns in the limited MATH dataset. With more self-generated samples, the model encounters repeated or low-quality negatives, potentially leading to spurious correlations. This is consistent with observations in rejection sampling, where excessive sampling can degrade performance.

\begin{table*}[!htbp]
\centering
\small
\resizebox{\textwidth}{!}{
\begin{tabular}{llcccccccc}
\toprule
\textbf{Model} & \textbf{Method} & \textbf{GSM8K} & \textbf{MATH} & \textbf{Minerva} & \textbf{Olympiad} & \textbf{AIME24} & \textbf{AMC23} & \textbf{College} & \textbf{Avg.} \\
\midrule
\multirow{6}{*}{\textbf{\shortstack{Qwen-2.5-Math-1.5B \\ (Mean@3)}}}
& Base & 41.6 & 35.2 & 9.3 & 22.4 & 10.0 & 32.5 & 8.3 & 22.8 \\
\cdashline{2-10}
& SFT & 56.3 & 43.3 & 11.6 & 16.5 & \textbf{8.9} & 32.5 & 20.3 & 27.1 \\
& \cellcolor[HTML]{E6E6FA}\textbf{+ RIFT} & \cellcolor[HTML]{E6E6FA}\textbf{74.7} & \cellcolor[HTML]{E6E6FA}\textbf{62.9} & \cellcolor[HTML]{E6E6FA}\textbf{18.5} & \cellcolor[HTML]{E6E6FA}\textbf{29.4} & \cellcolor[HTML]{E6E6FA}\textbf{8.9} & \cellcolor[HTML]{E6E6FA}\textbf{39.2} & \cellcolor[HTML]{E6E6FA}\textbf{36.8} & \cellcolor[HTML]{E6E6FA}\textbf{\underline{38.6}} \textcolor{red!50!black}{\textbf{(+11.5)}} \\
& DFT & \textbf{77.1} & 54.4 & 16.3 & 19.1 & 2.2 & 25.8 & \textbf{35.9} & 33.0 \\
& \cellcolor[HTML]{E6E6FA}\textbf{+ RIFT} & \cellcolor[HTML]{E6E6FA}74.7 & \cellcolor[HTML]{E6E6FA}\textbf{61.8} & \cellcolor[HTML]{E6E6FA}\textbf{18.3} & \cellcolor[HTML]{E6E6FA}\textbf{30.4} & \cellcolor[HTML]{E6E6FA}\textbf{4.4} & \cellcolor[HTML]{E6E6FA}\textbf{45.0} & \cellcolor[HTML]{E6E6FA}35.5 & \cellcolor[HTML]{E6E6FA}\textbf{\underline{38.6}} \textcolor{red!50!black}{\textbf{(+5.6)}} \\
& RIFT & 73.2 & 59.6 & \textbf{18.5} & 28.8 & \textbf{11.1} & \textbf{43.3} & 27.6 & \textbf{37.5} \\
& \cellcolor[HTML]{E6E6FA}\textbf{+ RIFT} & \cellcolor[HTML]{E6E6FA}\textbf{74.8} & \cellcolor[HTML]{E6E6FA}\textbf{62.3} & \cellcolor[HTML]{E6E6FA}17.5 & 
\cellcolor[HTML]{E6E6FA}\textbf{29.8} & \cellcolor[HTML]{E6E6FA}10.0 & 
\cellcolor[HTML]{E6E6FA}42.5 & 
\cellcolor[HTML]{E6E6FA}\textbf{34.1} & \cellcolor[HTML]{E6E6FA}\textbf{\underline{38.7}} \textcolor{red!50!black}{\textbf{(+1.2)}} \\
\midrule
\multirow{7}{*}{\textbf{\shortstack{Qwen-2.5-Math-1.5B \\ (Pass@3)}}}
& Base & 69.0 & 58.1 & 13.5 & 37.2 & 20.0 & 52.5 & 17.8 & 38.3 \\
\cdashline{2-10}
& SFT & 84.2 & 67.1 & 23.2 & 31.6 & \textbf{16.7} & \textbf{60.0} & 36.6 & 45.6 \\
& \cellcolor[HTML]{E6E6FA}\textbf{+ RIFT} & \cellcolor[HTML]{E6E6FA}\textbf{90.0} & \cellcolor[HTML]{E6E6FA}\textbf{78.7} & \cellcolor[HTML]{E6E6FA}\textbf{34.6} & \cellcolor[HTML]{E6E6FA}\textbf{42.4} & \cellcolor[HTML]{E6E6FA}13.3 & 
\cellcolor[HTML]{E6E6FA}\textbf{60.0} & \cellcolor[HTML]{E6E6FA}\textbf{47.4} & \cellcolor[HTML]{E6E6FA}\textbf{\underline{52.3}} \textcolor{red!50!black}{\textbf{(+6.7)}} \\
& DFT & \textbf{89.9} & 68.4 & 29.0 & 31.1 & 6.7 & 40.0 & 45.5 & 44.4 \\
& \cellcolor[HTML]{E6E6FA}\textbf{+ RIFT} & \cellcolor[HTML]{E6E6FA}87.9 & \cellcolor[HTML]{E6E6FA}\textbf{78.5} & \cellcolor[HTML]{E6E6FA}\textbf{29.4} & \cellcolor[HTML]{E6E6FA}\textbf{44.3} & \cellcolor[HTML]{E6E6FA}\textbf{13.3} & \cellcolor[HTML]{E6E6FA}\textbf{70.0} & \cellcolor[HTML]{E6E6FA}\textbf{46.7} & \cellcolor[HTML]{E6E6FA}\textbf{\underline{52.9}} \textcolor{red!50!black}{\textbf{(+8.5)}} \\
& RIFT & 88.6 & 77.3 & \textbf{33.5} & 41.9 & \textbf{20.0} & \textbf{70.0} & 38.6 & \textbf{\underline{52.8}} \\
& \cellcolor[HTML]{E6E6FA}\textbf{+ RIFT} & \cellcolor[HTML]{E6E6FA}\textbf{88.8} & \cellcolor[HTML]{E6E6FA}\textbf{79.3} & \cellcolor[HTML]{E6E6FA}28.7 & 
\cellcolor[HTML]{E6E6FA}\textbf{44.0} & \cellcolor[HTML]{E6E6FA}16.7 & 
\cellcolor[HTML]{E6E6FA}65.0 & \cellcolor[HTML]{E6E6FA}\textbf{45.8} & \cellcolor[HTML]{E6E6FA}{52.6} \textcolor{green!50!black}{\textbf{(-0.2)}} \\
\bottomrule
\end{tabular}
}
\caption{Mean@3 and Pass@3 accuracy (\%) on 7 mathematical benchmarks for Qwen-2.5-Math-1.5B under different sequential training protocols,
starting from models trained via SFT, DFT, or RIFT, we generate self-sampled responses and apply a second round of RIFT (denoted ``+ RIFT'').
Best results are in \textbf{bold}. \textcolor{red!50!black}{\textbf{(+)}} indicates the absolute improvement over the respective single-phase baseline.
}
\label{tab:plug_and_play}
\end{table*}

\subsection{Exploration Study: RIFT Drives Policy Convergence}

To further evaluate whether RIFT serves as a modular enhancement or constitutes an indispensable component in the traiing pipeline, we conduct a comparison study using Qwen-2.5-Math-1.5B trained on MATH, under three distinct protocols: (i) SFT followed by RIFT, (ii) DFT followed by RIFT, and (iii) Iterative RIFT (RIFT followed by RIFT). 
In each setting, the previously trained SFT, DFT, or  RIFT models serves as the base policy to generate a new corpus of self-generated responses. These responses are then utilized to perform a subsequent round of RIFT training.
The central hypothesis is that if RIFT functions as a plug-and-play refiner, performance should vary with initialization quality; conversely, if RIFT itself drives capability gains, final performance should converge across initialization strategies.

\begin{table}[htbp]
\centering
\resizebox{\columnwidth}{!}{
\renewcommand{\arraystretch}{0.95}
\begin{tabular}{lccccc}
\toprule
$K$ &  \# Num. & \# Total & \parbox{1.2cm}{\centering \# Pos. \\ (r$>$0)} & \parbox{1.4cm}{\centering \# Neg. \\ (r$<$0)} & \% Pos.  \\
\midrule
SFT & 3,000 & 24,000 & 8,300 & 15,700 & 65.4\%  \\
DFT & 3,000 & 24,000 & 7,964 & 16,859 & 70.2\%  \\
RIFT & 3,000 & 24,000 & 15,941 & 16,287 & 67.9\%  \\
\bottomrule
\end{tabular}
} 
\caption{Statistics of self-generated responses across previously trained SFT, DFT and RIFT models.}
\vspace{-2mm}
\label{tab:reapply_rollout_data_stats}
\end{table}

\begin{table}[htbp]
\centering
\small
\resizebox{\columnwidth}{!}{
\renewcommand{\arraystretch}{0.95}
\begin{tabular}{lccc}
\toprule
$K$ & \# Num. & \# Mixed-Reward Num. & \% Mixed \\
\midrule
SFT & 3,000 & 2,686 & 89.5\% \\
DFT & 3,000 & 2,697 & 89.9\% \\
RIFT & 3,000 & 2,669 & 89.0\% \\
\bottomrule
\end{tabular}
} 
\caption{Proportion of problems containing both positive- and negative-reward responses in self-generated responses across previously trained SFT, DFT and RIFT models.}
\vspace{-3mm}
\label{tab:reapply_data_stats_mixed}
\end{table}

As shown in Table~\ref{tab:plug_and_play}, the results support the hypothesis that RIFT itself drives capability gains.
Across all previously trained models, applying a second round of RIFT yields consistent improvements in Mean@3, outperforming their single-round counterparts: notable gains of +11.5 (SFT), +5.6 (DFT), and +1.2 (RIFT), even when starting from a strong RIFT-initialized policy. 
Gains diminish as the base policy strengthens, indicating convergence toward a shared high-reward policy.
However, the re-application of RIFT primarily boosts Mean@3 rather than Pass@3, which shows little to no further improvement compared to single-round trained RIFT. 
Together, these findings demonstrate that RIFT acts not as a plug-and-play refiner whose efficacy depends on initialization, but as a self-convergent alignment phase that capable of steering diverse initial policies toward comparable final performance.
Tables \ref{tab:reapply_rollout_data_stats} and \ref{tab:reapply_data_stats_mixed} provide a detailed statistical overview of the self-generated responses produced by the previously trained SFT, DFT, and RIFT models.

\section{Discussion on Loss Function and Gradient Stability}
\label{appendix:loss_discussion}

In this section, we provide a detailed analysis of our choice of using a Taylor approximation for the loss function, particularly in the context of penalizing negative samples.

\subsection{Design Intuition: Stability over Global Precision}
The primary motivation for employing a Taylor-based linear penalty (i.e., $(1-p)$) rather than the standard negative log-likelihood ($-\log(1-p)$) for negative samples lies in the \textit{bounded and stable gradients} it provides across the entire probability range $[0, 1]$. While the Taylor approximation is locally accurate near $p \to 1$, its global behavior is empirically more conducive to stable optimization in model alignment.

For negative trajectories, our design intuition follows a specific priority:
\begin{itemize}
    \item \textbf{Confident Errors:} Tokens assigned high probability ($p \to 1$) on negative samples indicate cases where the model is ``confidently wrong.'' These require strong, linear penalties to steer the model away from incorrect reasoning paths.
    \item \textbf{Correct Uncertainty:} Tokens with low probability ($p \to 0$) on negative examples reflect appropriate uncertainty. In such cases, the linear approximation provides a gentler correction signal, preventing the model from over-correcting on already unlikely tokens.
\end{itemize}

\subsection{Comparison with Alternative Formulations}
We evaluated several alternative functions to penalize negative samples, but found them less suitable for the following reasons:
\begin{itemize}
    \item \textbf{Logarithmic Loss ($-\log(1-p)$):} This formulation diverges as $p \to 1$. In practical training, this leads to gradient instability and a significant magnitude imbalance between positive and negative loss components, often resulting in training divergence.
    \item \textbf{Sigmoid-based Variants:} While sigmoid transformations can bound the loss, they introduce unnecessary saturation regions (where gradients vanish) without offering clear performance benefits in our preliminary experiments.
\end{itemize}

In summary, the choice of the Taylor approximation in RIFT is empirically motivated by its ability to maintain robust gradient signals and prevent optimization instability. We acknowledge that exploring other functional forms with similar bounded properties remains a promising direction for future work.

\subsection{Proofs and Theoretical Analysis in Section \ref{sec:methodology}}
\label{sec:proof}

\begin{proof}[Proof of Theorem \ref{thm:collapse}]
Let the dataset $\mathcal{D}$ be finite. The naive signed-weighted loss function can be explicitly written as:
\begin{equation}
\mathcal{L}_{\text{naive}}(\theta) = - \frac{1}{|\mathcal{D}|} \sum_{(x, y) \in \mathcal{D}} r(x, y) \log \pi_\theta(y \mid x).
\end{equation}
By the theorem's premise, the subset of negative samples $\mathcal{D}^- = \{(x, y) \in \mathcal{D} \mid r(x, y) < 0\}$ is non-empty. Let $(x_0, y_0) \in \mathcal{D}^-$ be a specific negative sample with weight $r_0 := r(x_0, y_0) < 0$.

We invoke the assumption of \textit{sufficient expressivity}, which implies that the model parameterization $\theta$ allows for the arbitrary manipulation of the probability mass $\pi_\theta(\cdot \mid x_0)$ on the support $\mathcal{Y}$. Specifically, we can construct a sequence of parameters $\{\theta_n\}_{n=1}^\infty$ such that the probability of the negative sample decays to zero:
\begin{equation}
\pi_{\theta_n}(y_0 \mid x_0) = \epsilon_n, \text{where } \epsilon_n > 0 \text{ and } \lim_{n \to \infty} \epsilon_n = 0.
\end{equation}
To ensure the well-posedness of the remaining terms, we stipulate that the probability mass removed from $y_0$ is redistributed to other tokens $y' \in \mathcal{Y} \setminus \{y_0\}$ uniformly, such that for all other samples $(x, y) \in \mathcal{D} \setminus \{(x_0, y_0)\}$, the probabilities satisfy $\pi_{\theta_n}(y \mid x) \geq \delta$ for some constant $\delta > 0$. This ensures that $\log \pi_{\theta_n}(y \mid x)$ remains bounded from below.

Then, we have 
\begin{equation}
\lim_{n \to \infty} \left| \frac{\partial \mathcal{L}_{\text{naive}}}{\partial \pi_{\theta_n}} \right| = \left| \frac{r}{\pi_{\theta_n} (y|x)} \right| = \infty.
\end{equation}

Now, we decompose the loss function for the sequence $\theta_n$:
\begin{equation}
\begin{aligned}
& \mathcal{L}_{\text{naive}}(\theta_n)  = - \frac{1}{|\mathcal{D}|}  \underbrace{r_0 \log \epsilon_n}_{T_1} \\ 
&\quad - \frac{1}{|\mathcal{D}|} \underbrace{\sum_{(x, y) \in \mathcal{D} \setminus \{(x_0, y_0)\}} r(x, y) \log \pi_{\theta_n}(y \mid x)}_{T_2} .
\end{aligned}
\end{equation}

We analyze the asymptotic behavior of the two terms $T_1$ and $T_2$ as $n \to \infty$:

\begin{enumerate}
\item \textbf{The Negative Sample Term ($T_1$):} 
Since $r_0 < 0$ and $\lim_{n \to \infty} \log \epsilon_n = -\infty$, the product behaves as:
\begin{equation}
\begin{aligned}
\lim_{n \to \infty} T_1 & = \lim_{n \to \infty} r_0 \log \epsilon_n \\ 
& = (-|r_0|) \cdot (-\infty) = +\infty.    
\end{aligned}
\end{equation}

\item \textbf{The Remaining Terms ($T_2$):} 
The sum $T_2$ consists of a finite number of terms.
\begin{itemize}
\item For any sample with $r(x, y) \geq 0$, since $\pi_{\theta_n}(y|x) \le 1$, we have $r(x, y) \log \pi_{\theta_n}(y|x) \le 0$. Thus, these terms are bounded from above by 0.
\item For any sample with $r(x, y) < 0$, since we enforced $\pi_{\theta_n}(y|x) \geq \delta$, the term $r(x, y) \log \pi_{\theta_n}(y|x)$ is finite.
\end{itemize}
Crucially, since we ensure other probabilities do not vanish (i.e., $\pi \ge \delta$), the logarithm $\log \pi$ is bounded below by $\log \delta$. Consequently, the entire sum $T_2$ is bounded, i.e., there exists a constant $M$ such that $|T_2| < M$.
\end{enumerate}

Combining these results, the limit of the total loss is:
\begin{equation}
\begin{aligned}
\lim_{n \to \infty} \mathcal{L}_{\text{naive}}(\theta_n) &= - \frac{1}{|\mathcal{D}|} \left( \lim_{n \to \infty} T_1 + \lim_{n \to \infty} T_2 \right) \\
&= - \frac{1}{|\mathcal{D}|} ( +\infty + O(1) ) \\
&= -\infty.
\end{aligned}
\end{equation}

Thus, we have identified a sequence in the parameter space along which the loss diverges to negative infinity. This proves that $\mathcal{L}_{\text{naive}}$ is unbounded from below.
\end{proof}

First, we give a formal reformulation of Theorem \ref{thm:RIFT_properties} as follows. 

\begin{theorem}[Formal Statement of Theorem \ref{thm:RIFT_properties}]

The RIFT formulation satisfies the following theoretical properties:
\begin{enumerate}
\item[(\romannumeral1)] \textbf{Boundedness:} Assume that the reward $r$ is bounded with constant $|r| \leq M$ for all the data. Then, the loss objective is bounded from below.
\item[(\romannumeral2)] \textbf{Reward Lower-Bound Maximization:} Let the expected reward objective be $\mathcal{J}(\theta) := \mathbb{E}_{y \sim \pi_\theta} [r(x, y)]$. Assume that all the data sampled from the reference distribution $\pi_{ref}$ has a lower bound probability, i.e., $\pi_{ref} (y | x) \geq C_2$. Then, there exists a constant $C_1 > 0$, such that $\mathcal{J} (\theta) \geq - \frac{1}{C_2} \mathcal{L}_{\text{RIFT}}(\theta) + C_2$. Thus, minimizing the RIFT loss objective is equivalent to maximize the reward objective.
\end{enumerate}
    
\end{theorem}

\begin{proof}[Proof of Theorem \ref{thm:RIFT_properties}]
\textbf{(\romannumeral1) Boundedness:}
Recall the formulation of the RIFT loss:
\begin{equation}
\begin{aligned}
\mathcal{L}_{\text{RIFT}}(\theta) = & - \mathbb{E}_{\mathcal{D}^+}  \left[ r(x, y) \log \pi_\theta (y \mid x) \right] \\
& + \mathbb{E}_{\mathcal{D}^-} \left[ r(x, y) \pi_\theta (y \mid x) \right].    
\end{aligned}
\end{equation}
We analyze the boundedness of the two terms separately. 

For the positive sampled term, we have 
\begin{equation}
- \mathbb{E}_{\mathcal{D}^+} \left[ r(x, y) \log \pi_\theta (y \mid x) \right] \geq 0,
\end{equation}
since $\log \pi_\theta \leq 0$ for the probability distribution.

For the negative sampled term, the absolute value is bounded by
\begin{equation}
\begin{aligned}
& |- \mathbb{E}_{\mathcal{D}^-} \left[ r(x, y) \pi_\theta (y \mid x) \right] | \\
\leq & \mathbb{E}_{\mathcal{D}^-} \left[ | r(x, y) \pi_\theta (y \mid x) | \right] \\ 
\leq & \mathbb{E}_{\mathcal{D}^-} \left[ | r(x, y) | \right] \leq M.    
\end{aligned}
\end{equation}
Hence, the second term is bounded from below.

Thus, $\mathcal{L}_{\text{RIFT}}$ is bounded from below.

\textbf{(\romannumeral2) Reward Lower-Bound Maximization:}
We aim to show that maximizing the negative RIFT loss (i.e., minimizing $\mathcal{L}_{\text{RIFT}}$) is equivalent to maximizing a surrogate lower bound of the expected reward $\mathcal{J}(\theta)$.

First, we rewrite the expected reward objective using Importance Sampling (IS) to shift the expectation from the policy distribution $\pi_\theta$ to the reference data distribution $\pi_{ref}$:
\begin{equation}
\begin{aligned}
\mathcal{J}(\theta) & = \mathbb{E}_{y \sim \pi_\theta} [r(x, y)] \\ 
&= \mathbb{E}_{y \sim \pi_{ref}} \left[ \frac{\pi_\theta(y|x)}{\pi_{ref}(y|x)} r(x, y) \right].    
\end{aligned}
\end{equation}
Let $\rho(y|x) = \frac{\pi_\theta(y|x)}{\pi_{ref}(y|x)}$ be the likelihood ratio. We utilize the fundamental inequality relating linear and logarithmic functions: for any $u > 0$, $\log u \le u - 1$, which implies $u \ge 1 + \log u$.

We decompose the objective into contributions from positive ($\mathcal{D}^+$) and negative ($\mathcal{D}^-$) domains:

\textbf{Case 1: Positive Samples ($r > 0$).}
Applying the inequality $\rho \ge 1 + \log \rho$:
\begin{equation}
\begin{aligned}
\mathbb{E}_{\mathcal{D}^+} [\rho \cdot r] &\ge \mathbb{E}_{\mathcal{D}^+} \left[ r (1 + \log \rho) \right] \\
&= \mathbb{E}_{\mathcal{D}^+} \left[ r \left(1 + \log \pi_\theta - \log \pi_{ref}\right) \right] \\
&= \mathbb{E}_{\mathcal{D}^+} [r \log \pi_\theta] + C_1,
\end{aligned}
\end{equation}
where $C_1 = \mathbb{E}_{\mathcal{D}^+}[r(1 - \log \pi_{ref})]$ is a constant with respect to $\theta$. This recovers the positive component of $-\mathcal{L}_{\text{RIFT}}$.

\textbf{Case 2: Negative Samples ($r < 0$).}
For negative samples, we seek a lower bound for the term $\rho \cdot r$. Note that $(x, y)$ are data sampled from the distribution $\pi_{ref}$, there exits a constant $C_2 > 0$, such that $\pi_{ref} (y | x) \geq C_2$ for all $(x, y)$. Thus, 
\begin{equation}
\mathbb{E}_{\mathcal{D}^-} [\rho \cdot r] \geq \frac{1}{C_2} \mathbb{E}_{\mathcal{D}^-} [\pi_\theta (y | x) r (x, y)]
\end{equation}

\textbf{Synthesis:}
Combining the results, we define a global surrogate objective $\mathcal{J}_{\text{surr}}$ (the IS-derived logarithmic lower bound for positive samples and the linear lower bound for negative samples):
\begin{equation}
\begin{aligned}
\mathcal{J}(\theta) \geq &\mathbb{E}_{\mathcal{D}^+} [r \log \pi_\theta] + \frac{1}{C_2}\mathbb{E}_{\mathcal{D}^-} [r \pi_\theta] + C_1 \\ \geq & -\frac{1}{C_1} \mathcal{L}_{\text{RIFT}}(\theta) + C_2.    
\end{aligned}
\end{equation}
Therefore, maximizing $-\mathcal{L}_{\text{RIFT}}$ (or minimizing $\mathcal{L}_{\text{RIFT}}$) effectively maximizes a rigorous surrogate lower bound of the true expected reward $\mathcal{J}(\theta)$.
\end{proof}

\end{document}